%% file: main.tex
\pdfoutput=1
\PassOptionsToPackage{shortlabels}{enumitem}
\documentclass[11pt, letterpaper, logo, onecolumn, copyright]{template_from_google}

\usepackage[T1]{fontenc}
\usepackage[utf8]{inputenc}
\usepackage{microtype}

\usepackage{amsmath}
\usepackage{amsfonts}
\usepackage{amssymb}
\usepackage{mathtools}
\usepackage{amsthm}
\usepackage{bm}
\usepackage{nicefrac}

\usepackage{booktabs}
\usepackage{enumitem}
\usepackage{fancyhdr}
\usepackage{multirow}
\usepackage{array}
\usepackage{wrapfig}
\usepackage{arydshln}
\usepackage{authblk}

\usepackage[ruled,vlined]{algorithm2e}

\usepackage{graphicx}
\usepackage{caption}
\usepackage{subcaption}
\usepackage{tikz}
\usetikzlibrary{shapes.geometric, arrows.meta, positioning, calc}

\usepackage[dvipsnames]{xcolor}

\usepackage[authoryear, sort&compress, round]{natbib}

\usepackage{hyperref}
\usepackage{url}
\definecolor{darkblue}{rgb}{0.0, 0.0, 0.6}
\definecolor{darkred}{rgb}{0.7, 0.0, 0.0}
\hypersetup{
  pdffitwindow=true,
  pdfstartview={FitH},
  pdfnewwindow=true,
  colorlinks,
  linktocpage=true,
  linkcolor=darkred,
  urlcolor=darkblue,
  citecolor=darkblue
}

\usepackage[export]{adjustbox}
\usepackage[capitalise]{cleveref}
\crefname{equation}{Eq.}{Eqs.}

\newcommand{\E}{\mathbb{E}}

\newcommand{\cfg}{\alpha_{\mathrm{cfg}}}
\newcommand{\diff}{\mathrm{d}}

\newtheorem{theorem}{Theorem}
\newtheorem{proposition}[theorem]{Proposition}

\newtheorem{example}[theorem]{Example}

\newtheorem{remark}[theorem]{Remark}

\makeatletter
\renewcommand\paragraph{\@startsection{paragraph}{4}{\z@}
            {-2.5ex\@plus -1ex \@minus -.25ex}
            {1.25ex \@plus .25ex}
            {\itshape\normalsize\bfseries}}
\makeatother

\let\cite\citep

\title{Diff-Instruct with Diffused Reward: Towards Principled One-step Generator RL}

\newcommand{\shorttitle}{Diff-Instruct with Diffused Reward}
\reportnumber{}
\renewcommand{\today}{}

\author[1]{Junyi Wu}
\author[2]{Weijian Luo}
\author[1]{Haoyang Zheng}
\author[1]{Ruizhe Zhang}
\author[1]{Guang Lin}
\affil[1]{Purdue University}
\affil[2]{Independent Researcher}

\begin{abstract}
Recent advances in one-step text-to-image generation have enabled real-time synthesis with remarkable efficiency and quality.
Previous reinforcement learning methods for one-step generators combines the image space reward optimization and the diffusion noisy space distribution matching. This paradigm brings challenges due to a mismatch between terminal reward optimization and the underlying generative dynamics. As a result, optimization tends to exploit stochastic degrees of freedom, often improving reward at the expense of image fidelity. To address this issue, we propose \textbf{Diff-Instruct with Diffused Reward} (\textsc{Didr}), a data-free trajectory-level alignment framework derived from Integral KL minimization. \textsc{Didr} propagates the RLHF-optimal reward-tilted clean-image distribution across all noise levels along the diffusion trajectory. We show that this objective admits the same minimizer as clean-image RLHF, while naturally inducing \textbf{the Diffused Reward Score} (DRS), which acts as a reward-driven correction to the reference score function. To make this practical, we further introduce \textbf{the Diffused Reward Proxy} (DRP), an efficient estimator of DRS based on differentiable short-step denoising. Extensive experiments demonstrate that \textsc{Didr} consistently Pareto-dominates existing one-step SDXL baselines. Moreover, when transferred to a 6B DiT backbone (\textsc{Z-Image}), \textsc{Didr} surpasses its 50-step teacher in preference alignment while requiring only a single generation step. \href{https://github.com/Junyi-W/DIDR}{[Code]}.

\end{abstract}

\begin{document}
\maketitle

\input{intro}
\input{back}
\input{method}
\input{experiments}
\input{related}
\input{conclusion}

\bibliography{refs}

\newpage

\appendix

\tableofcontents

\input{derivation}

\newpage
\end{document}

%% file: intro.tex
\section{Introduction}

Deep generative models have achieved remarkable progress in text-to-image (T2I) synthesis~\citep{rombach2022high,saharia2022photorealistic,ramesh2022hierarchical,esser2024scaling,nichol2021improved,karras2020analyzing}, driving the development of one-step generators~\citep{Luo2023DiffInstructAU,yin2024improved,zhou2024long,sauer2023stylegan,kang2023gigagan,zheng2024diffusion,liu2023instaflow,xu2024ufogen} that map latent noise directly to images in a single forward pass via diffusion distillation~\citep{luo2023comprehensive} and GAN-based techniques~\citep{goodfellow2014generative,sauer2023adversarial,zheng2025ultra}. While these models achieve real-time synthesis, they frequently exhibit suboptimal aesthetic quality and poor alignment with human preferences---a central requirement for real-world deployment.

Diff-Instruct\citep{luo2023diffinstruct} opens the door to train one-step image generators by distribution matching distillation\citep{yin2023one} by minimizing an Integral Kullback-Leibler divergence. Later works based on concepts of distribution matching explored the reinforcement learning post-training of one-step text-to-image generators by combining image-space rewards with latent space distribution matching~\citep{anonymous2024diffinstruct,luo2024onestep}. However, these attempts lack a trajectory-level target directly induced by the KL-regularized RLHF optimum. These methods apply preference signals exclusively at the clean image output $x_0$ while imposing KL regularization over the full noisy trajectory. This creates a structural mismatch: the reward acts at the clean-image level while the regularizer spans all noise levels, so the combined objective does not correspond to any proper RLHF-induced target on the generator distribution. We refer to the resulting tendency to over-tilt toward high-reward outputs as \textbf{terminal reward domination}, and make it precise in a tractable example (\S\ref{subsec:trd}). In contrast, \textsc{Didr} derives a principled trajectory-level target directly from the KL-regularized RLHF optimum, ensuring that reward and regularization are balanced at every noise level.

To close this gap, we propose \textbf{Diff-Instruct with Diffused Reward} ({\textsc{Didr}}), a principled reinforcement learning framework for one-step image generators without using image data. Our core insight is to systematically propagate reward signals defined on clean images to every noise level of the diffusion trajectory. We realize this by starting from the KL-regularized RLHF objective, identifying its optimal clean-image target as the reward-tilted density $q^*(x_0|c) \propto q_0(x_0|c)\exp(r(x_0,c)/\tau)$, and diffusing this target through the reference process to obtain trajectory-level reward-tilted marginals $q_t^*(x_t|c)$. Minimizing the Integral KL (IKL) divergence against these marginals yields a target score that decomposes into the reference score plus a reward-induced correction, the \textbf{Diffused Reward Score (DRS)}, defined at every noise level through reference posterior.

\begin{figure}[!t]
\centering
\includegraphics[width=0.95\linewidth]{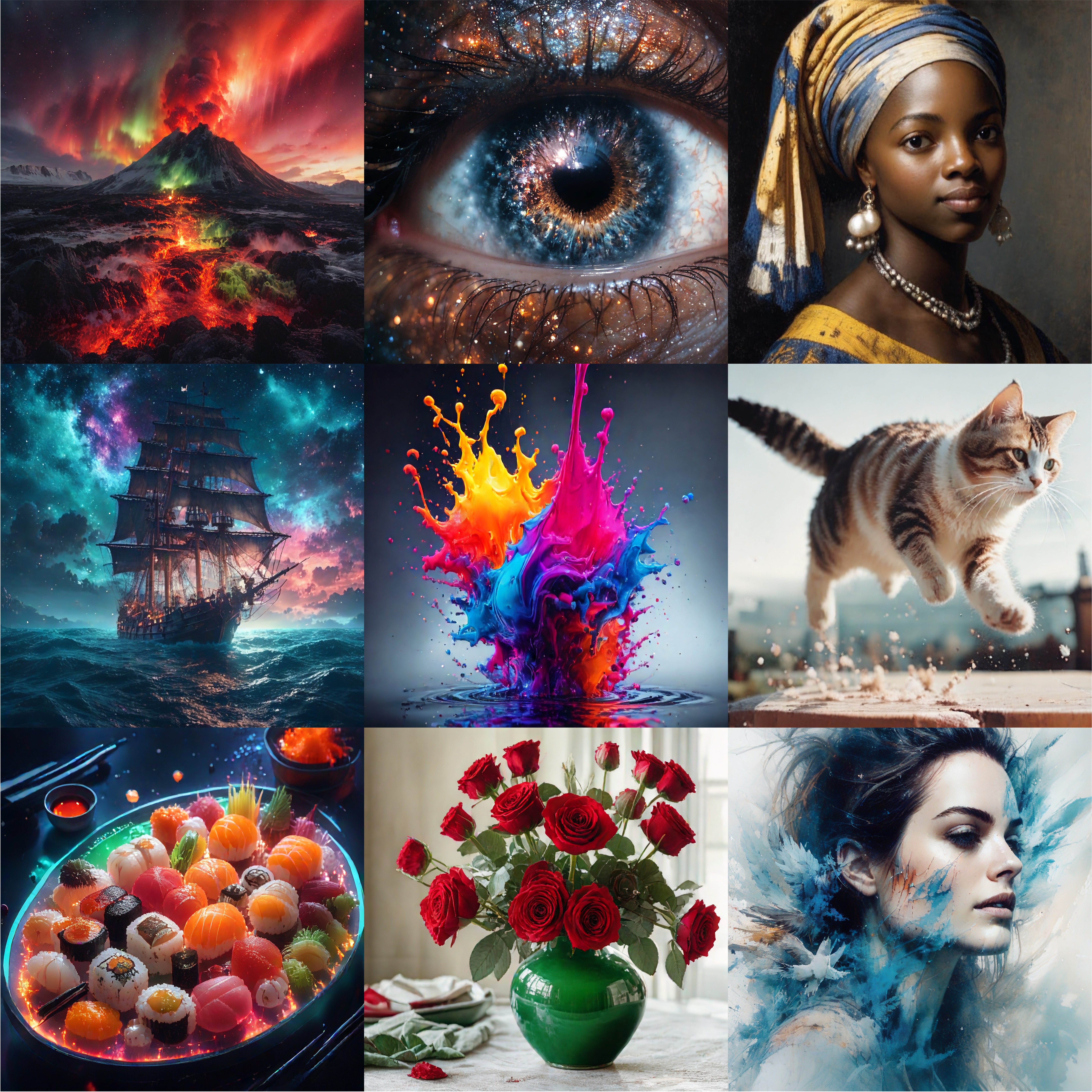}
\caption{$1024\times 1024$ images from one-step SDXL aligned via \textsc{Didr}. Prompts in Appendix~\ref{app:prompts}.}
\label{fig:teaser}
\end{figure}

Since the DRS involves an intractable posterior expectation, we further derive the \textbf{Diffused Reward Proxy (DRP)}: a differentiable estimator that runs short-step denoising chains from $x_t$ using the frozen reference model, propagating reward gradients stably back through the trajectory to yield a practical, data-free algorithm. Empirically, compared with Diff-Instruct*~\citep{luo2024onestep} (the strongest prior one-step SDXL baseline), \textsc{Didr}$_{\text{longer}}$ improves PickScore from $23.1$ to $23.9$ and ImageReward from $1.01$ to $1.10$ ($+8.9\%$); the standard \textsc{Didr} achieves the best PickScore--FID trade-off among one-step methods. We also scales \textsc{Didr} to the 6B Z-Image backbone, surpassing the 50-step base in a single step. We summarize our contributions as follows:
\begin{itemize}[leftmargin=6pt, noitemsep, topsep=2pt]

        \item We identify \emph{terminal reward domination} (a structural mismatch in prior endpoint-only methods where optimizers exploit noise to collapse onto high-reward regions) and propose \textsc{Didr}, a principled trajectory-level alignment framework. Grounded in our proof of Integral KL equivalence to KL-regularized RLHF, \textsc{Didr} rigorously propagates the RLHF-optimal, reward-tilted clean-image density forward to define a principled alignment target at \emph{every} noise level.

    \item We analytically derive {DRS}, which exactly decomposes the trajectory target score into a reference score plus a reward-induced correction. To make this tractable, we introduce {DRP}, a zero-data estimator that computes this correction via differentiable short-step posterior denoising through a frozen reference model, enabling stable reward-gradient propagation across the entire trajectory.

    \item Extensive experiments demonstrate that \textsc{Didr} sets a new state-of-the-art for one-step alignment without requiring any image training data. On one-step SDXL at $1024 \times 1024$, standard \textsc{Didr} Pareto-dominates all existing one-step baselines in the PickScore--FID trade-off, while \textsc{Didr}$_{\text{longer}}$ pushes PickScore to $23.9$ and ImageReward from $1.01$ to $1.10$ ($+8.9\%$). Furthermore, \textsc{Didr} successfully transfers to a 6B DiT-based Z-Image backbone, surpassing its 50-step teacher on preference metrics in a single step.
\end{itemize}

%% file: back.tex
\section{Preliminaries}\label{sec:background}

\textbf{Diffusion Trajectories and One-step Generators.} Let $x_0$ denote a clean image drawn from the reference data distribution $q_0(x_0|c)$, conditioned on a text prompt $c$. A diffusion model defines a forward noising process that gradually corrupts $x_0$ into pure noise over a continuous time variable $t \in [0, T]$. This is characterized by a transition kernel $q_t(x_t|x_0)$, which induces the noisy marginal distributions $q_t(x_t|c) = \int q_t(x_t|x_0)q_0(x_0|c)\diff x_0$. To reverse this process for image generation, one needs the Stein score of the marginals, $\nabla_{x_t}\log q_t(x_t|c)$, a vector field pointing toward higher data density. In practice, a neural network $s_{\mathrm{ref}}(x_t, t, c)$ is trained to approximate this score via Denoising Score Matching (DSM)~\citep{vincent2011connection,song2021scorebased}. In our framework, this trained multi-step diffusion model serves as the \emph{reference model} and is kept frozen.

While the reference model generates high-quality images, its iterative sampling process is computationally expensive. To achieve real-time synthesis, a one-step generator~\citep{goodfellow2014generative,song2023consistency} bypasses the iterative ODE/SDE integration entirely. It maps a standard Gaussian noise vector $z \sim p_z$ directly to a clean image in a single forward pass: $x_0 = g_\theta(z, c)$. This mapping induces an implicit clean-image distribution, denoted as $p_{\theta,0}(x_0|c)$. By hypothetically injecting the same forward noise into these generated samples, we can define the generator's noisy trajectory $p_{\theta,t}(x_t|c) = \int q_t(x_t|x_0)p_{\theta,0}(x_0|c)\diff x_0$. This distinction between the reference trajectory $q_t$ and the generator trajectory $p_{\theta,t}$ is crucial, as our objective relies on matching their reward-tilted dynamics.

\textbf{Trajectory Distillation and Terminal-reward Alignment.}
Standard one-step distillation, such as Diff-Instruct~\citep{Luo2023DiffInstructAU}, trains one-step generators by minimizing an Integral KL (IKL) divergence across the entire diffusion trajectory, anchoring the generator's noisy marginals to the bare reference marginals $\{q_t\}_{t\in[0,T]}$. To incorporate human preferences, recent methods like Diff-Instruct* (DI*,~\citep{luo2024onestep}) and Diff-Instruct++ (DI++,~\citep{anonymous2024diffinstruct}) extend this objective by naively appending a terminal reward. For a generic divergence $\mathcal{D}$, the objective becomes:
\begin{equation}
\mathcal{L}_{\mathrm{term}}(\theta)
= -\E_{c,\, x_0 \sim p_{\theta,0}(x_0|c)}[r(x_0,c)]
+ \tau\int_0^T w(t)\,\mathcal{D}\!\bigl(p_{\theta,t}(x_t|c)\|q_t(x_t|c)\bigr)\diff t.
\label{eq:lterm}
\end{equation}
Critically, this formulation evaluates the reward exclusively at the clean-image endpoint ($x_0$), while applying KL regularization against the unrewarded reference marginals ($\{q_t\}$) across all noise levels. This structural mismatch creates a fundamental optimization loophole: because forward noise naturally masks image details, the regularization penalty against mode divergence drastically weakens at high noise levels. Optimizers readily exploit this trajectory vulnerability, abandoning the reference distribution to collapse onto high-reward regions without incurring sufficient KL penalty. We formally identify this failure mode as \emph{terminal reward domination}. Resolving this inherent misalignment necessitates a principled trajectory-level target, which directly motivates our construction in Section~\ref{sec:method}.

\textbf{The Principled KL-Regularized Target.}
To construct a mathematically sound trajectory objective, we must first establish the correct alignment target at the clean-image level. Standard Reinforcement Learning from Human Feedback (RLHF)~\citep{christiano2017deep,ouyang2022training} formulates preference optimization by balancing reward maximization against a KL-divergence penalty to explicitly prevent fidelity degradation:
\begin{equation}
\mathcal{L}(\theta) = \E_{c,\, x_0 \sim p_\theta(x_0|c)} \bigl[-r(x_0, c)\bigr] + \tau\, \mathcal{D}_{\mathrm{KL}}\!\bigl(p_\theta(x_0|c) \| q_0(x_0|c)\bigr),
\label{eq:rlhf_kl}
\end{equation}
where the temperature $\tau>0$ governs the trade-off between semantic preference and reference fidelity. Crucially, this objective admits a unique, closed-form global optimum---the reward-tilted density:
\begin{equation}
q^*(x_0|c) = \frac{1}{Z(c)}\, q_0(x_0|c)\exp\!\left(\frac{r(x_0,c)}{\tau}\right),
\label{eq:reward_tilted_target}
\end{equation}
where $Z(c)$ is the normalizing partition function. As detailed in Appendix~\ref{app:proof_lem1}, minimizing Eq.~\eqref{eq:rlhf_kl} is mathematically equivalent to directly minimizing $\tau\mathcal{D}_{\mathrm{KL}}(p_\theta\|q^*)$. Thus, $q^*$ represents the fundamentally correct, perfectly balanced alignment target for clean images. Our core insight, developed next, is that this optimal target can be rigorously propagated forward through the diffusion process to guide every intermediate noise level without structural mismatch.

%% file: method.tex
\section{The Proposed Diff-Instruct with Diffused Reward}
\label{sec:method}

\textsc{Didr} resolves the target mismatch of prior methods through four steps.
(i)~We illustrate how the structural mismatch in $\mathcal{L}_{\mathrm{term}}$ can lead to over-tilting via a tractable example (\S\ref{subsec:trd}).
(ii)~We diffuse the RLHF-optimal clean-image target $q^*$ through the reference process to obtain principled trajectory-level marginals $\{q_t^*\}$, and show that minimizing the IKL against $\{q_t^*\}$ is equivalent to KL-regularized RLHF (\S\ref{subsec:target}).
(iii)~We derive the target score as the reference score plus a reward-induced correction (DRS) defined at every noise level (\S\ref{subsec:target}).
(iv)~We approximate DRS with DRP via differentiable short-step denoising, and train the generator using a Teaching Assistant (TA) score model (\S\ref{sec:drp}--\S\ref{subsec:algorithm}).

\subsection{Terminal Reward Domination}
\label{subsec:trd}
As noise increases, the forward diffusion process smooths modal structure, weakening the KL cost of suppressing an unrewarded mode. Because the reward acts only at the clean level while this weakening occurs across the full trajectory, the reward signal may dominate: the optimizer can over-tilt toward high-reward outputs relative to the RLHF target $q^*$. We refer to this behavior as \emph{terminal reward domination} and prove it analytically in the tractable example below.

\begin{example}[Terminal reward domination]
\label{prop:reward_domination}
Let $p_0=\tfrac{1}{2}\mathcal{N}(-\mu,\sigma^2)+\tfrac{1}{2}\mathcal{N}(\mu,\sigma^2)$ with reward $r(x_0)=\mathbf{1}[x_0>0]$. While $q^*$ remains a proper mixture for all finite $\tau$, $\mathcal{L}_{\mathrm{term}}$ fully collapses onto the rewarded mode whenever $\tau\le\gamma(1-\sigma^2)/[2\mu^2(-\log\sigma^2)]$, a threshold that need not be small (Appendix~\ref{app:reward_domination_gaussian}, Figure~\ref{fig:schematic_combined}).
\end{example}
This theoretical result shows that $q^*$ (Eq.~\eqref{eq:reward_tilted_target}) is a strictly better alignment target: it stays calibrated for all $\tau$ precisely because reward and KL regularization are balanced at the same level. Section~\ref{subsec:target} formalizes how to lift $q^*$ to a full trajectory objective.

\begin{figure}[!t]
\centering
\begin{subfigure}[t]{0.48\linewidth}
  \centering
  \includegraphics[width=\linewidth]{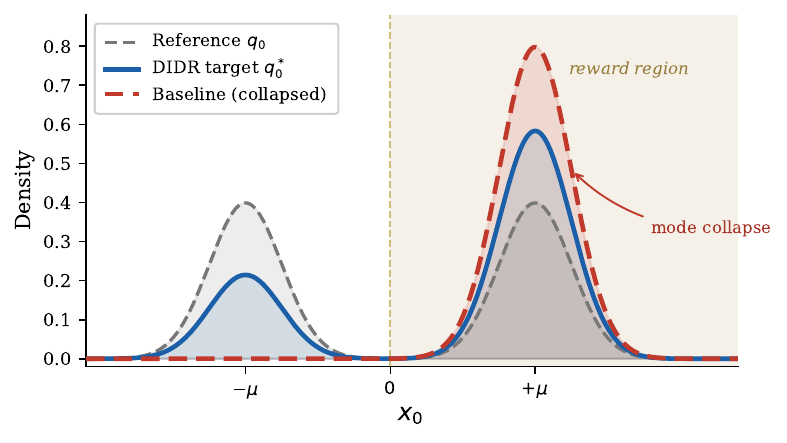}
  \caption{\textbf{Theoretical prediction} (\S\ref{subsec:trd}, $\tau=1$): $\mathcal{L}_{\mathrm{term}}$ collapses onto the rewarded mode (red) while $q^*$ remains a soft mixture (blue).}
\end{subfigure}
\hfill
\begin{subfigure}[t]{0.48\linewidth}
  \centering
  \includegraphics[width=\linewidth]{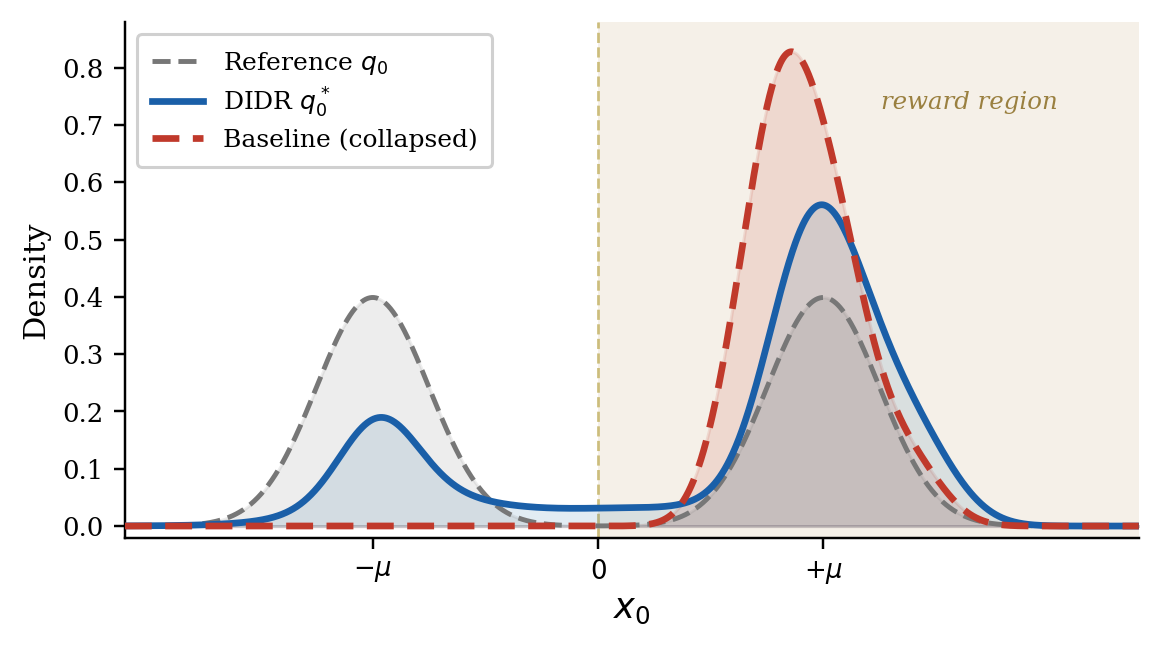}
  \caption{\textbf{1-D empirical validation} (setup in Appendix~\ref{app:toy_experiment}): \textsc{Didr} vs.\ DI++ as the baseline.}
\end{subfigure}
\vspace{-0.05in}
\caption{\textbf{Terminal reward domination.} The reward signal in $\mathcal{L}_{\mathrm{term}}$ tends to dominate the KL regularizer, causing the optimizer to over-tilt toward the high-reward mode relative to the balanced RLHF target $q^*$. \textbf{(a)} plots the theoretical prediction from \S\ref{subsec:trd}; \textbf{(b)} shows the 1-D empirical result (Appendix~\ref{app:toy_experiment}), confirming the collapse in practice even at $\tau=1$.}
\label{fig:schematic_combined}
\vspace{-0.1in}
\end{figure}

\subsection{Reward-Tilted Trajectory Objective}
\label{subsec:target}
Our core construction is the reward-tilted trajectory $\{q_t^*\}$, a principled optimization target defined at every noise level. We obtain it by lifting the RLHF optimum $q^*(x_0|c)$ (Eq.~\eqref{eq:reward_tilted_target}) to all noise levels via the reference forward process:
\begin{equation}
q_t^*(x_t|c)
=
\int q_t(x_t|x_0)\,q^*(x_0|c)\,\diff x_0 .
\label{eq:qt_star}
\end{equation}
The family $\{q_t^*\}_{t\in[0,T]}$ is the diffusion trajectory of the RLHF minimizer $q^*$, providing a well-defined target at every noise level. Proposition~\ref{thm:ikl_equiv} establishes the formal equivalence.

\begin{figure}[!t]
\vspace{-0.1in}
\centering
\begin{subfigure}[t]{0.48\linewidth}
  \centering
  \includegraphics[width=\linewidth]{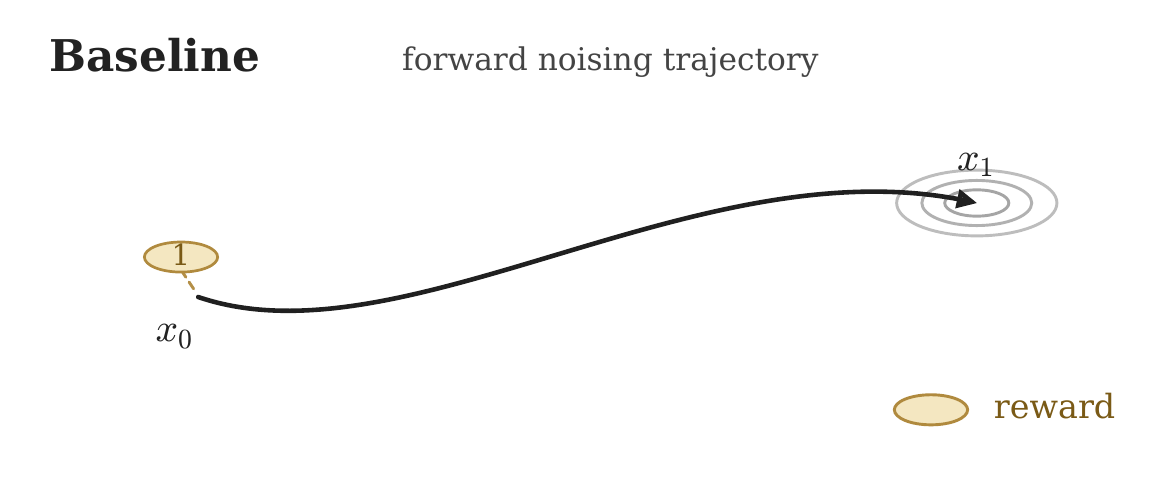}
  \caption{Terminal-reward alignment: reward only at $x_0$.}
\end{subfigure}
\hfill
\begin{subfigure}[t]{0.48\linewidth}
  \centering
  \includegraphics[width=\linewidth]{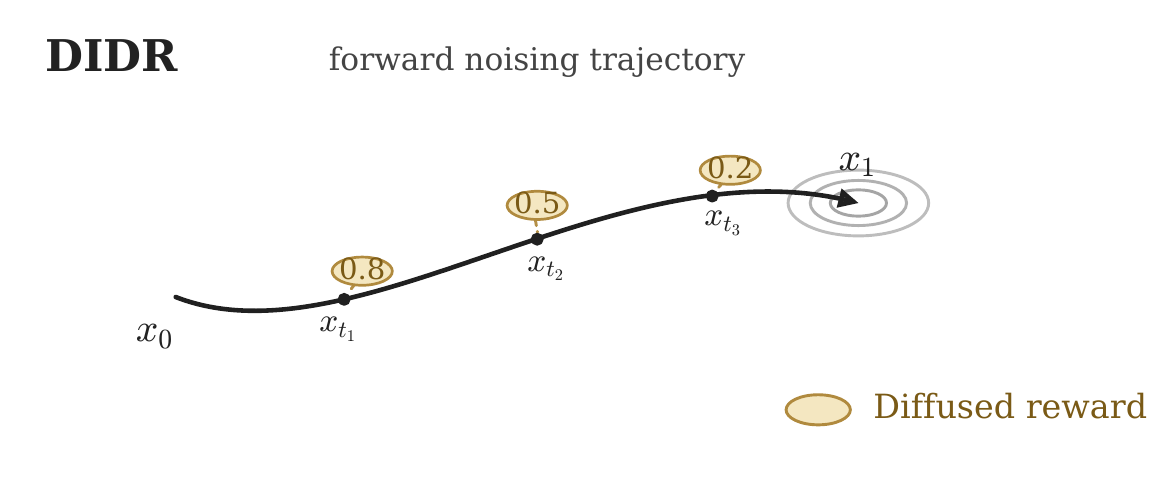}
  \caption{\textsc{Didr}: reward at every $x_t$ via DRS.}
\end{subfigure}
\caption{\textbf{Reward propagation comparison.} Prior methods inject reward only at the clean endpoint $x_0$ while regularizing against bare reference marginals. \textsc{Didr} propagates preference through the full trajectory via the Diffused Reward Score.}
\label{fig:reward_comparison}
\end{figure}

\textbf{Objective.}
We train the generator by minimizing the IKL divergence between its forward-noised marginals $p_{\theta,t}$ and the reward-tilted trajectory $\{q_t^*\}$:
\begin{equation}
\mathcal{L}_{\mathrm{DIDR}}(\theta)
=
\E_{c\sim\mathcal{C}}
\left[
\int_0^T w(t)\,
\mathcal{D}_{\mathrm{KL}}\!\bigl(p_{\theta,t}(\cdot|c)\|q_t^*(\cdot|c)\bigr)\diff t
\right].
\label{eq:ikl_obj}
\end{equation}
By minimizing IKL against $q_t^*$ rather than the bare reference marginals, \textsc{Didr} encodes the full RLHF signal inside the trajectory without a separate terminal reward term, inheriting the broad distribution support and stable gradient flow of the IKL formulation while correctly injecting reward signals at every noise level.

\begin{proposition}[IKL shares the RLHF minimizer]
\label{thm:ikl_equiv}
For any clean distribution $p_0(\cdot|c)$ and its forward marginals $p_t(\cdot|c)$, if $Z(c)<\infty$ and $w(t)>0$ a.e., then
\[
\operatorname*{arg\,min}_{p_0}\mathcal{L}_{\mathrm{RLHF}}(p_0|c)
=
\operatorname*{arg\,min}_{p_0}\mathcal{L}_{\mathrm{DIDR}}(p_0|c)
=
\{q^*(\cdot|c)\}.
\]
\end{proposition}
(See Appendix~\ref{app:proof_thm_ikl_equiv} for proof.) The shared minimizer shows that the ideal IKL objective targets the RLHF optimum $q^*$. Note that this equivalence holds at the distribution level; the practical algorithm introduces score, posterior, and finite-sample approximations. Nevertheless, $\{q_t^*\}$ provides a well-defined optimization target throughout training.

\begin{theorem}[Score-based DIDR gradient]
\label{thm:score_gradient}
Differentiating the Integral KL through the generator and the forward noising process gives the ideal score-based gradient
\begin{equation}
\begin{aligned}
\nabla_\theta\mathcal{L}_{\mathrm{DIDR}}
=
\E_{\substack{c\sim\mathcal{C},\,z\sim p_z,\,t\sim\pi(t)\\
x_0=g_\theta(z,c)\\
x_t\sim q_t(\cdot|x_0)}}
\Bigg[
w(t)
\bigl(
s_\theta(x_t,t,c)
- s_{\mathrm{ref}}(x_t,t,c)
- s_r(x_t,t,c)
\bigr)
\frac{\partial x_t}{\partial\theta}
\Bigg].
\end{aligned}
\label{eq:ikl_grad}
\end{equation}
where $s_\theta(x_t,t,c)=\nabla_{x_t}\log p_{\theta,t}(x_t|c)$.
Also, $s_{\mathrm{ref}}(x_t,t,c)=\nabla_{x_t}\log q_t(x_t|c)$. The Diffused Reward Score (DRS) is
\begin{equation}
s_r(x_t,t,c)
=
\nabla_{x_t}\log
\E_{x_0\sim q(x_0|x_t,c)}
\left[
\exp\!\left(\frac{r(x_0,c)}{\tau}\right)
\right].
\label{eq:sr_definition}
\end{equation}
\end{theorem}
(See Appendix~\ref{app:proof_thm2} for proof.) The DRS term $s_r$ adds a reward-driven correction to the reference score, steering the generator toward the RLHF optimum at every noise level. Because $s_r$ is defined through the posterior $q(x_0|x_t,c)$, it propagates preference information only where the clean image is identifiable from the noisy observation, naturally attenuating reward guidance at high noise. The generator update is therefore driven by the mismatch between its own score and the reward-tilted target score $s_{\mathrm{ref}}+s_r$.

\begin{remark}[Connection to classifier guidance]
\label{rem:classifier_guidance}
Classifier guidance~\citep{dhariwal2021diffusion} is the first-order approximation of the DRS; the DRS keeps the full log-exponential tilt.
\end{remark}

\subsection{Diffused Reward Proxy}
\label{sec:drp}
Although the DRS provides a principled reward signal at every noise level, computing Eq.~\eqref{eq:sr_definition} requires the posterior $q(x_0|x_t,c)$, which is unavailable in closed form for a learned diffusion model. Assuming posterior samples admit a reparameterization $x_0 = \mathcal{G}(x_t,\boldsymbol{\epsilon},c)$ differentiable in $x_t$, differentiating the log-normalizer yields the pathwise form
\begin{equation}
\begin{aligned}
s_r(x_t,t,c)
=
\E_{x_0\sim q(x_0|x_t,c)}
\left[
\frac{\exp\!\left(\frac{r(x_0,c)}{\tau}\right)}
{\E_{\tilde{x}_0\sim q(\tilde{x}_0|x_t,c)}
\left[\exp\!\left(\frac{r(\tilde{x}_0,c)}{\tau}\right)\right]}
\cdot
\frac{1}{\tau}\nabla_{x_t}r(x_0,c)
\right],
\end{aligned}
\label{eq:drs_pathwise}
\end{equation}
where $\nabla_{x_t}r(x_0,c)$ is the gradient through a posterior sample $x_0$ treated as a function of $x_t$: the DRS is a reward-softmax-weighted average of pathwise reward gradients. We make this tractable by approximating the unavailable posterior with differentiable short-step denoising chains from $x_t$ under the frozen reference model, $\hat{x}_0 = \mathcal{G}_{\mathrm{ref}}(x_t,\boldsymbol{\epsilon},c)$, giving DRP formalized below.

\begin{proposition}[Diffused Reward Proxy]
\label{prop:drp}
Let $\hat{x}_0^{(k)}$ be $K$ independent $S$-step denoising samples from the frozen reference model,
\begin{equation}
\hat{x}_0^{(k)}
=
\mathcal{G}_{\mathrm{ref}}(x_t,\boldsymbol{\epsilon}^{(k)},c),
\qquad
k=1,\ldots,K,
\qquad
\boldsymbol{\epsilon}^{(k)}\sim p(\boldsymbol{\epsilon}) .
\end{equation}
For fixed $\boldsymbol{\epsilon}$, each chain is differentiable with respect to $x_t$. The DRP approximates the unavailable posterior expectation in Eq.~\eqref{eq:drs_pathwise} by the Monte Carlo estimator
\begin{equation}
\widehat{s}_r(x_t,t,c)
=
\frac{1}{\tau}
\sum_{k=1}^K
\omega^{(k)}
\nabla_{x_t}r(\hat{x}_0^{(k)},c),
\qquad
\omega^{(k)}
=
\frac{\exp(r(\hat{x}_0^{(k)},c)/\tau)}
{\sum_j \exp(r(\hat{x}_0^{(j)},c)/\tau)} .
\label{eq:drp_estimator}
\end{equation}
Here the pathwise reward gradient is
\begin{equation}
\nabla_{x_t}r(\hat{x}_0^{(k)},c)
=
\left(
\frac{\partial \mathcal{G}_{\mathrm{ref}}(x_t,\boldsymbol{\epsilon}^{(k)},c)}
{\partial x_t}
\right)^{\!\top}
\nabla_{\hat{x}_0}r(\hat{x}_0^{(k)},c).
\end{equation}
\end{proposition}
(See Appendix~\ref{app:proof_lem4} for derivation.) The softmax weights $\omega^{(k)}$ concentrate gradient mass on high-reward denoised images, while the pathwise Jacobian $\partial\mathcal{G}_{\mathrm{ref}}/\partial x_t$ propagates clean-image reward information back to $x_t$ through the frozen denoising chain without storing reference model activations.

\begin{remark}[Why the proxy works]
At low noise, the posterior $q(x_0|x_t,c)$ is concentrated, so short denoising chains yield accurate proxies. At high noise, $x_t$ carries little information about $x_0$, so the DRS correction is less sensitive to the accuracy of individual posterior samples; a short denoising proxy captures the leading-order correction.
\end{remark}

\vspace{-3mm}
\subsection{Practical Training with a Teaching Assistant}
\label{subsec:algorithm}
\vspace{-3mm}
Two substitutions make Eq.~\eqref{eq:ikl_grad} implementable. First, the generator score $s_\theta = \nabla_{x_t}\log p_{\theta,t}$ is unavailable for an implicit one-step generator; following~\citet{Luo2023DiffInstructAU}, we surrogate it with a Teaching Assistant (TA) score model $s_\psi$ trained concurrently via denoising score matching (DSM)~\citep{vincent2011connection,song2021scorebased} on noisy generator samples, so that $s_\psi(x_t,t,c)\approx s_\theta(x_t,t,c)$. Second, for text-to-image generation we replace the reference score $s_{\mathrm{ref}}$ with the CFG-corrected score
\begin{equation}
\tilde{s}_{\mathrm{ref}}(x_t, t, c) = s_{\mathrm{ref}}(x_t, t, \emptyset) + \cfg \bigl[ s_{\mathrm{ref}}(x_t, t, c) - s_{\mathrm{ref}}(x_t, t, \emptyset) \bigr],
\label{eq:cfg}
\end{equation}
where $\cfg\geq 1$ amplifies the prompt-conditional component.

\begin{remark}[Implicit and explicit reward signals]
The CFG-corrected reference score $\tilde{s}_{\mathrm{ref}}$ can be viewed as an implicit reward-tilted score induced by text guidance, while the DRS $s_r$ is an explicit correction from the chosen reward model. These two signals are complementary and can be combined, giving \textsc{Didr} flexibility in the reward design.
\end{remark}

With these substitutions, the practical generator update is
\begin{equation}
\operatorname{Grad}(\theta)
=
\E_{\substack{c\sim\mathcal{C},\,z\sim p_z,\,t\\
x_0=g_\theta(z,c),\,x_t\sim q_t(\cdot|x_0)}}
\left[
w(t)
\left(
s_\psi(x_t,t,c)
-
\tilde{s}_{\mathrm{ref}}(x_t,t,c)
-
s_r(x_t,t,c)
\right)
\frac{\partial x_t}{\partial\theta}
\right].
\label{eq:grad_practical}
\end{equation}

\textsc{Didr} alternates two stages (Algorithm~\ref{alg:didr}, Figure~\ref{fig:workflow}). \textbf{Stage I} keeps $s_\psi$ synchronized with the evolving $g_\theta$: sample $x_0 = g_\theta(z,c)$, diffuse to $x_t$, and update $s_\psi$ via DSM. \textbf{Stage II} updates $g_\theta$: draw $K$ differentiable $S$-step denoising chains from $x_t$, compute the DRP weights, and backpropagate Eq.~\eqref{eq:grad_practical} to $\theta$. Both the reference diffusion model and the reward model remain frozen throughout.

\vspace{-3mm}

%% file: experiments.tex
\section{Empirical Results}
\vspace{-3mm}

\textbf{Models.}
We evaluate \textsc{Didr} on two backbones. For primary experiments, we apply \textsc{Didr} to one-step SDXL~\citep{podell2023sdxl} at $1024\times1024$ resolution, initializing the generator from the DMD2-SDXL-1step checkpoint~\citep{yin2024improved} and using the pretrained SDXL diffusion model as both the frozen reference $s_{\mathrm{ref}}$ and the TA initialization. We report a standard variant (\textsc{Didr}) and an extended-training variant (\textsc{Didr}$_{\text{longer}}$). To assess generalizability, we additionally apply \textsc{Didr} to the Z-Image backbone~\citep{zimage} (\textsc{zimage-Didr}), initializing from Z-Image-Turbo. Full hyperparameters are in Appendix~\ref{app:implementation}.

\textbf{Training Data and Reward.}
All models are trained on text prompts from LAION-Aesthetic-6.25+~\citep{zhou2024long}; no image data is required. As the reward model $r(x_0,c)$, we use PickScore~\citep{pickscore}, a widely accepted human preference model trained on large-scale pairwise human annotations that scores how well a generated image matches human preference given a text prompt. The training reward uses the raw unscaled PickScore, which is $1/98.86$ of the official reported values in Table~\ref{tab:main}, making $\tau{=}0.01$ a reasonable temperature for the DRP softmax.

\textbf{Evaluation.}
We report three categories of metrics: \textit{Preference} --- PickScore~\citep{pickscore}, ImageReward~\citep{xu2023imagereward}, HPSv2.1~\citep{hpsv2}, and Aesthetic Score~\citep{laionaes}; \textit{Text alignment} --- CLIPScore~\citep{hessel2021clipscore,radford2021learning}, DPG-Bench~\citep{dpgbench}, and GenEval~\citep{geneval}; and \textit{Fidelity} --- FID~\citep{heusel2017gans}. PickScore, ImageReward, Aesthetic Score, and CLIPScore are evaluated on $1\text{k}$ prompts from the MSCOCO-2017 validation set~\citep{lin2014microsoft}; FID on the full $30\text{k}$ validation images; HPSv2.1, DPG-Bench, and GenEval each use their own standard benchmark prompts. Since PickScore is also our training reward, the remaining preference metrics serve as independent out-of-domain evaluations. Further details are provided in Appendix~\ref{app:evaluation_metrics}.

\vspace{-3mm}
\subsection{Quantitative Results}

\textbf{\textsc{Didr}$_{\text{longer}}$ achieves the highest PickScore among evaluated one-step text-to-image models.}
As shown in Table~\ref{tab:main}, \textsc{Didr}$_{\text{longer}}$ achieves the highest PickScore ($23.9$) and surpasses all multi-step reference models on all four preference metrics under our evaluation protocol, using only 2.6B parameters and a single inference step. Gains span both the training reward and independent out-of-domain metrics (ImageReward $+8.9\%$, HPSv2.1 $+0.6$, Aesthetic Score $+0.15$ over DI*), suggesting that alignment does not simply overfit to PickScore. Per-category GenEval and HPSv2.1 results are provided in Tables~\ref{tab:geneval} and~\ref{tab:hpsv2} (Appendix~\ref{app:geneval}).

\textbf{Standard \textsc{Didr} Pareto-dominates all one-step alignment baselines on preference vs.\ fidelity.}
\begin{figure}[!t]
\centering
\begin{subfigure}[b]{0.48\linewidth}
  \centering
  \includegraphics[width=0.85\linewidth]{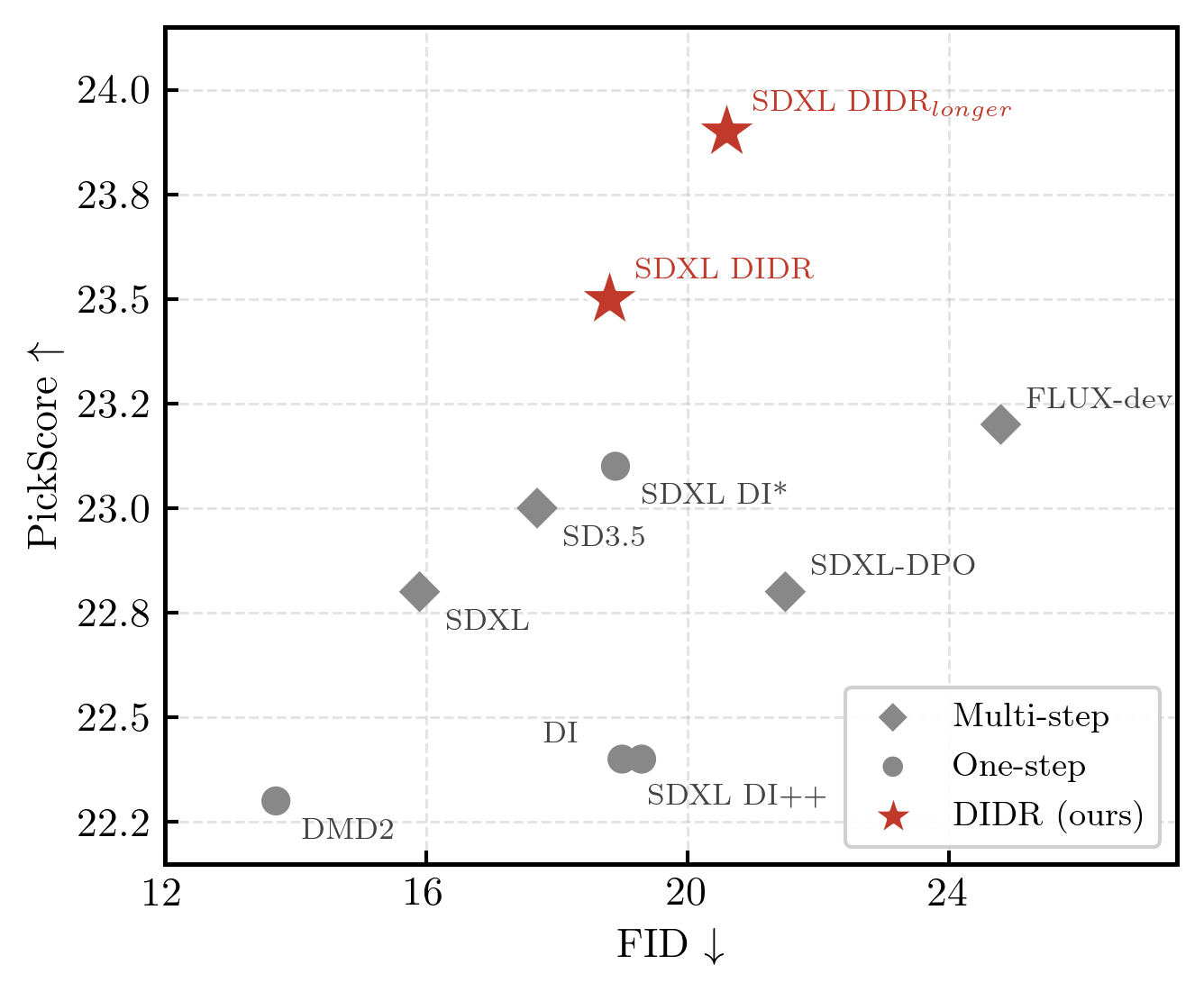}
  \phantomcaption
  \label{fig:pareto}
\end{subfigure}
\hfill
\begin{subfigure}[b]{0.48\linewidth}
  \centering
  \includegraphics[width=0.85\linewidth]{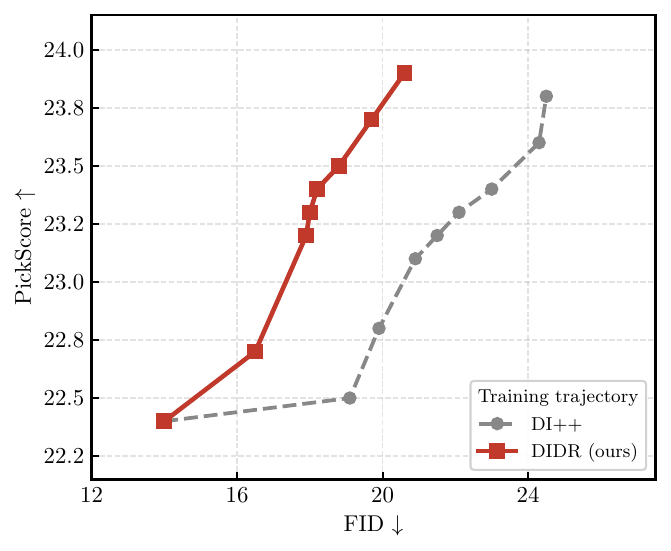}
  \phantomcaption
  \label{fig:trajectory}
\end{subfigure}
\caption{\textbf{PickScore--FID analysis}. \textit{Left}: final model comparison on MSCOCO-2017; \textsc{Didr} Pareto-dominates all one-step alignment baselines. \textit{Right}: training trajectories, \textsc{Didr} vs.\ Diff-Instruct++.}
\label{fig:pickscore_fid}
\vspace{-0.1in}
\end{figure}
Since the target distribution is the reward-tilted $q^*$ (Eq.~\eqref{eq:reward_tilted_target}), some FID increase is theoretically expected; the key question is which method best navigates this trade-off. As shown in Figure~\ref{fig:pareto}, \textsc{Didr} (PickScore $23.5$, FID $18.8$) Pareto-dominates all one-step alignment baselines and simultaneously surpasses FLUX-dev and SDXL-DPO on both axes. \textsc{Didr} also maintains text alignment, achieving the highest DPG-Bench ($75.06$) and GenEval ($0.579$) among one-step SDXL methods, with only a marginal CLIPScore dip, which is expected since text-alignment metrics are not directly optimized by the reward model.

\textbf{\textsc{Didr} generalizes across architectures and model scales.}
\textsc{Didr} generalizes to DiT architectures at larger scale: applied to the 6B-parameter Z-Image backbone, \textsc{zimage-Didr} surpasses the 50-step Z-Image base on all preference metrics in a single step. Against the 8-step Z-Image-Turbo, it leads on ImageReward, Aesthetic Score, CLIPScore, and DPG-Bench with lower FID, while trailing on PickScore, HPSv2.1, and GenEval. These gains are driven by training, not initialization---the 1-step Z-Image-Turbo initially scores ImageReward $0.39$ and FID $36.7$, vs $1.08$ and $22.1$ for \textsc{zimage-Didr}.

\input{tabs/main_results}

\vspace{-3mm}
\subsection{Qualitative Comparison}
\vspace{-3mm}

\begin{figure}[!t]
\centering
\includegraphics[width=0.571\textwidth]{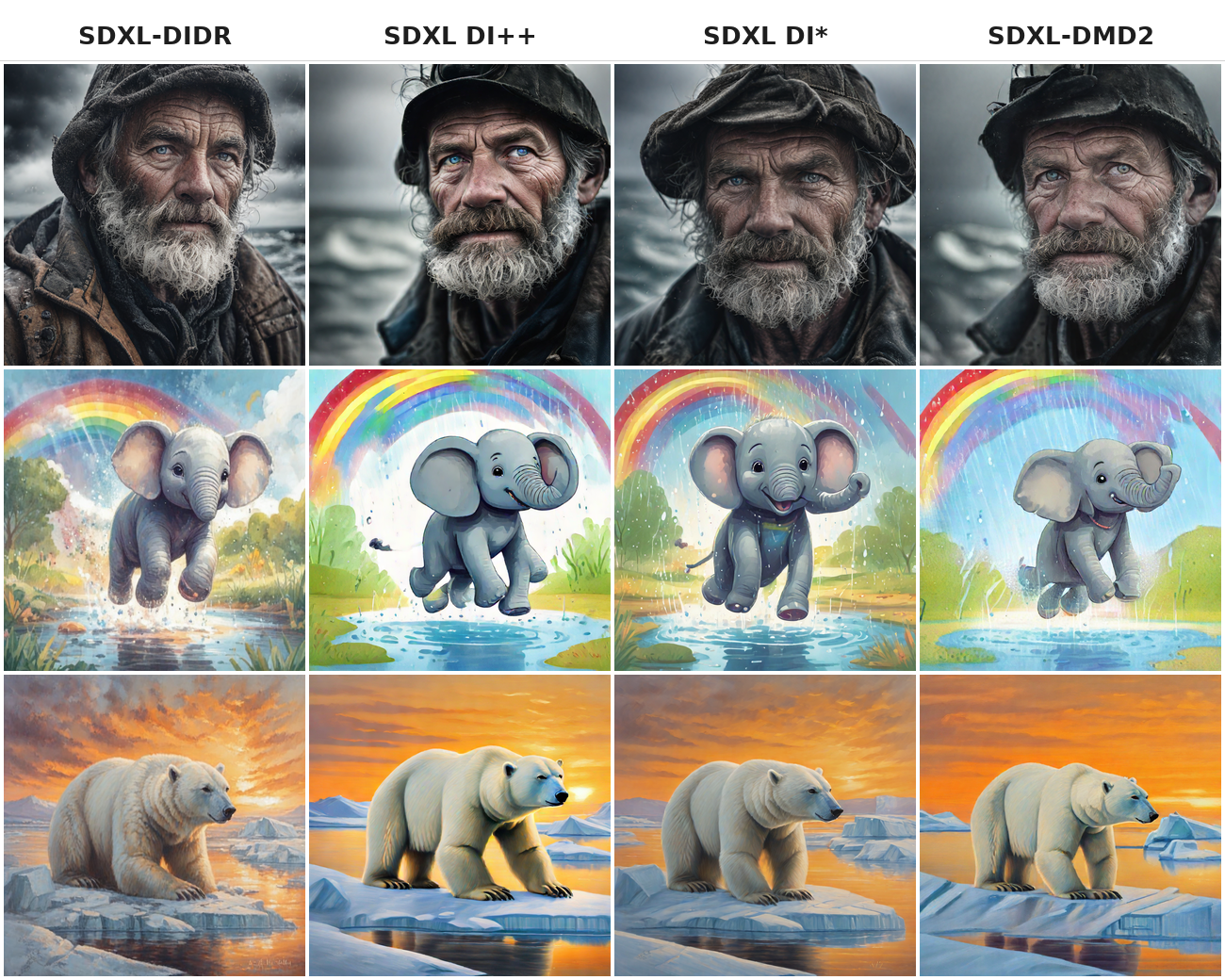}
\hfill
\includegraphics[width=0.428\textwidth]{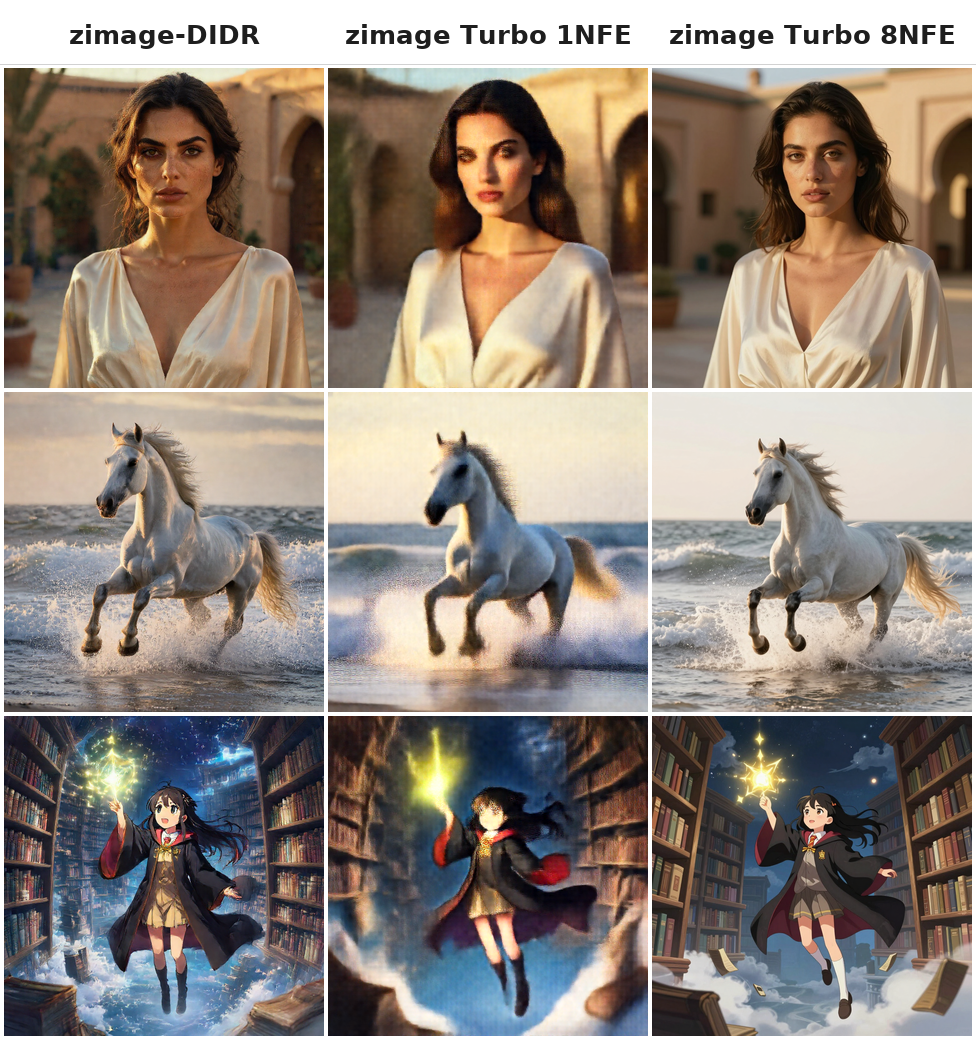}
\caption{\textbf{Qualitative comparison at $1024\times1024$.} \textit{Left}: SDXL backbone (SDXL-DIDR, DI++, DI*, DMD2). \textit{Right}: Z-Image backbone (zimage-DIDR, Turbo 1NFE, Turbo 8NFE). Rows: backbone-specific prompts listed in Appendix~\ref{app:prompts}.}
\label{fig:qualitative}
\vspace{-0.1in}
\end{figure}

\textsc{Didr} consistently improves visual quality over prior one-step methods while preserving perceptual naturalness (Figures~\ref{fig:teaser} and~\ref{fig:qualitative}). Compared with DI++ and DI*, \textsc{Didr} produces sharper fine detail and more natural color---skin tones and textures are faithfully rendered rather than over-saturated or painterly, reflecting a more favorable reward--fidelity trade-off. On the Z-Image backbone, \textsc{zimage-Didr} recovers dramatically from the severely degraded 1-step Turbo initialization and matches 8-step Z-Image-Turbo in perceptual sharpness and composition with a single step. A qualitative comparison against multi-step baselines (SDXL, Z-Image, FLUX-dev, SD3.5-Large) is provided in Figure~\ref{fig:onestep_vs_multistep}.

\vspace{-3mm}
\subsection{Ablation Studies}
\label{sec:ablation}
\vspace{-3mm}

\textbf{Effect of DRS.}
Figure~\ref{fig:trajectory} shows that \textsc{Didr} Pareto-dominates DI++ at every training checkpoint, confirming that trajectory-level reward propagation provides a structural advantage over terminal reward alone.

\textbf{Effect of $(K,S)$ and temperature $\tau$.}
Both $K$ and $S$ contribute independently: raising either from 1 to 4 improves all metrics, and $(K,S){=}(4,4)$ achieves the best result on all four metrics (Table~\ref{tab:ablation_ks}). Decreasing $\tau$ monotonically improves PickScore ($22.6{\to}23.8$) at the cost of FID ($16.3{\to}22.2$); we select $\tau{=}0.01$ as the operating point balancing both (Table~\ref{tab:ablation_tau}, Figure~\ref{fig:tau_ablation}).

\input{tabs/ablation}

\vspace{-3mm}

%% file: tabs/main_results.tex
\begin{table*}[!t]
\centering
\caption{Comparison of \textsc{Didr} against baselines at $1024\times1024$ resolution. Multi-step models are shown as reference only (no bold). \textbf{Bold}: best result within each one-step group. $\uparrow$/$\downarrow$: higher/lower is better. $^\dagger$: Z-Image-Turbo at 1~NFE (unofficial; \textsc{Didr} training initialization only).}
\label{tab:main}
\resizebox{\textwidth}{!}{\begin{tabular}{lclccccccccc}
\toprule
& & & & \multicolumn{4}{c}{\textbf{Preference} $\uparrow$} & \multicolumn{3}{c}{\textbf{Text Alignment} $\uparrow$} & \\
\cmidrule(lr){5-8} \cmidrule(lr){9-11}
\textbf{Model} & \textbf{Steps} & \textbf{Arch.} & \textbf{Params} & \textbf{PickS.} & \textbf{ImgR.} & \textbf{HPSv2.1} & \textbf{Aesth.} & \textbf{CLIP} & \textbf{DPG} & \textbf{GenEval} & \textbf{FID} $\downarrow$ \\
\midrule
\multicolumn{12}{l}{\textit{Multi-step reference}} \\
SDXL~\citep{podell2023sdxl}            & 50 & UNet & 2.6B & 22.8 & 0.82 & 28.98 & 5.45 & 33.86 & 74.42 & 0.551 & 15.9 \\
SDXL-DPO~\citep{wallace2024diffusion}  & 50 & UNet & 2.6B & 22.8 & 0.92 & 30.45 & 5.61 & 33.97 & 74.24 & 0.549 & 21.5 \\
SD3.5-large~\citep{sd35}               & 28 & DiT  & 8B   & 23.0 & 1.01 & 30.06 & 5.42 & 33.85 & 84.88 & 0.717 & 17.7 \\
FLUX-dev~\citep{fluxdev}               & 50 & DiT  & 12B  & 23.2 & 1.05 & 30.65 & 5.54 & 32.88 & 83.86 & 0.672 & 24.8 \\
Z-Image~\citep{zimage}                 & 50 & DiT  & 6B   & 22.5 & 0.99 & 30.57 & 5.43 & 33.15 & 86.30 & 0.840 & 14.3 \\
\midrule
\multicolumn{12}{l}{\textit{One-step SDXL}} \\
SDXL DMD2~\citep{yin2024improved}                  & 1 & UNet & 2.6B & 22.3 & 0.83 & 29.99 & 5.47 & 33.31 & 74.01 & 0.557 & \textbf{13.7} \\
SDXL Diff-Instruct~\citep{Luo2023DiffInstructAU}   & 1 & UNet & 2.6B & 22.4 & 0.92 & 31.51 & 5.53 & 33.35 & 74.98 & 0.560 & 19.3 \\
SDXL Diff-Instruct++~\citep{anonymous2024diffinstruct} & 1 & UNet & 2.6B & 22.4 & 0.92 & 31.44 & 5.58 & \textbf{33.37} & 75.01 & 0.560 & 19.0 \\
SDXL Diff-Instruct*~\citep{luo2024onestep}          & 1 & UNet & 2.6B & 23.1 & 1.01 & 33.29 & 5.68 & 33.09 & 74.97 & 0.555 & 18.9 \\
\textbf{SDXL \textsc{Didr} (Ours)}                  & 1 & UNet & 2.6B & 23.5 & 1.04 & 33.77 & 5.82 & 33.03 & \textbf{75.06} & \textbf{0.579} & 18.8 \\
\textbf{SDXL \textsc{Didr}$_{\text{longer}}$ (Ours)} & 1 & UNet & 2.6B & \textbf{23.9} & \textbf{1.10} & \textbf{33.89} & \textbf{5.83} & 32.97 & 74.99 & 0.566 & 20.6 \\
\midrule
\multicolumn{12}{l}{\textit{Z-Image backbone}} \\
Z-Image-Turbo~\citep{zimage}                        & 8 & DiT & 6B & \textbf{23.0} & 1.01 & \textbf{31.86} & 5.39 & 33.08 & 85.00 & \textbf{0.750} & 25.2 \\
\textbf{Z-Image \textsc{Didr} (Ours)}               & 1 & DiT & 6B & 22.6 & \textbf{1.08} & 31.40 & \textbf{5.46} & \textbf{33.35} & \textbf{86.63} & 0.692 & \textbf{22.1} \\
\hdashline
Z-Image-Turbo$^\dagger$~\citep{zimage}              & 1 & DiT & 6B & 21.0 & 0.39 & 24.54 & 4.68 & 33.28 & 77.05 & 0.572 & 36.7 \\
\bottomrule
\end{tabular}}
\vspace{-0.15in}
\end{table*}

%% file: tabs/ablation.tex
\begin{table}[h]

\centering
\small
\begin{minipage}[t]{0.44\linewidth}
\centering
\setlength{\tabcolsep}{4pt}
\caption{\textbf{Effect of $(K,S)$.} Both axes contribute independently; $(K,S){=}(4,4)$ is best. $^\dagger$: default; \textbf{bold}: best in column.}
\begin{tabular}{cccccc}
\toprule
$K$ & $S$ & \textbf{PickS.} $\uparrow$ & \textbf{ImgR.} $\uparrow$ & \textbf{Aesth.} $\uparrow$ & \textbf{FID} $\downarrow$ \\
\midrule
1          & 1          & 22.6 & 0.80 & 5.32 & 24.1 \\
1          & 2          & 22.8 & 0.86 & 5.41 & 22.4 \\
1          & 4          & 23.1 & 0.92 & 5.66 & 20.8 \\
2          & 1          & 22.7 & 0.84 & 5.44 & 23.5 \\
4          & 1          & 22.8 & 0.90 & 5.50 & 22.0 \\
$4^\dagger$ & $4^\dagger$ & \textbf{23.5} & \textbf{1.04} & \textbf{5.82} & \textbf{18.8} \\
\bottomrule
\end{tabular}
\label{tab:ablation_ks}
\end{minipage}\ \
\hfill
\begin{minipage}[t]{0.53\linewidth}
\centering
\setlength{\tabcolsep}{5pt}
\caption{\textbf{Effect of $\tau$.} Lower $\tau$ improves preference at the cost of FID; $\tau{=}0.01$ is the operating point. $(K,S){=}(4,4)$ fixed. $^\dagger$: default; \textbf{bold}: best in column.}
\begin{tabular}{lcccc}
\toprule
$\tau$ & \textbf{PickS.} $\uparrow$ & \textbf{ImgR.} $\uparrow$ & \textbf{Aesth.} $\uparrow$ & \textbf{FID} $\downarrow$ \\
\midrule
$1.00$             & 22.6          & 0.88          & 5.56          & \textbf{16.3} \\
$0.50$             & 22.9          & 0.93          & 5.41          & 17.1 \\
$0.10$             & 22.9          & 1.02          & 5.68          & 17.5 \\
$0.05$             & 23.2          & 1.01          & 5.74          & 17.8 \\
$0.01^\dagger$     & 23.5          & 1.04          & 5.82          & 18.8 \\
$0.001$            & \textbf{23.8} & \textbf{1.17} & \textbf{5.83} & 22.2 \\
\bottomrule
\end{tabular}
\label{tab:ablation_tau}
\end{minipage}

\end{table}

%% file: related.tex
\section{Related Work}
\vspace{-3mm}
\textbf{RLHF and Preference Alignment for Diffusion Models.}
Preference alignment for multi-step diffusion models uses trajectory-accessible signals via supervised fine-tuning~\citep{dai2023emu,podell2023sdxl}, reward backpropagation~\citep{prabhudesai2023aligning,clark2023directly,black2023training,fan2024reinforcement}, or offline preference optimization~\citep{wallace2024diffusion,yang2024using,hong2024margin}. These methods rely on explicit denoising trajectories and do not directly address one-step generators with implicit distributions; \textsc{Didr} derives a compatible trajectory objective that applies to implicit generators without requiring trajectory access during inference.

\textbf{Classifier Guidance and Reward-Guided Sampling.}
Classifier guidance~\citep{dhariwal2021diffusion} adds $\nabla_{x_t}\log p(c|x_t)$ to the score, which is a first-order approximation of the DRS (Remark~\ref{rem:classifier_guidance}), and related inference-time methods~\citep{chung2022diffusion,prabhudesai2023aligning,clark2023directly,kong2026ultra} apply differentiable rewards along the trajectory. However, all such methods modify only the sampling procedure and leave the generator weights unchanged; \textsc{Didr} instead trains an aligned one-step generator.

\vspace{-3mm}

%% file: conclusion.tex
\section{Conclusion and Limitations}
\vspace{-3mm}

\textbf{Conclusion.} We identified terminal reward domination as a fundamental failure mode of trajectory distillation with terminal rewards, and proposed \textsc{Didr} to resolve it. The core idea is to lift the RLHF-optimal clean-image target $q^*$ to a principled trajectory objective via the Diffused Reward Score, and approximate it tractably with the Diffused Reward Proxy through differentiable short-step denoising---requiring no image training data. \textsc{Didr} achieves the best PickScore--FID trade-off among one-step alignment methods, improves all four preference metrics over the strongest prior baseline, and transfers to the 6B Z-Image backbone, surpassing its 50-step teacher on preference metrics in a single step. These results indicate that principled trajectory-level reward propagation is an effective alignment strategy across architectures and scales.

\textbf{Limitations.} Since DRP propagates gradients through the reward model, it may amplify reward-model biases or trigger reward hacking when the reward model is imperfect. Each generator update also requires $K \times S$ differentiable reference-denoising steps, increasing training cost relative to endpoint-only methods. Finally, counting, fine-grained structure, and complex multi-object scenes remain open challenges for single-step generation (Figure~\ref{fig:failure}). Extending \textsc{Didr} to multi-step generators and developing reward-ensemble strategies to mitigate hacking are natural directions for future work.

%% file: derivation.tex
\section{Method Supplements}
\begin{figure}[t]
\centering
\includegraphics[width=0.95\linewidth]{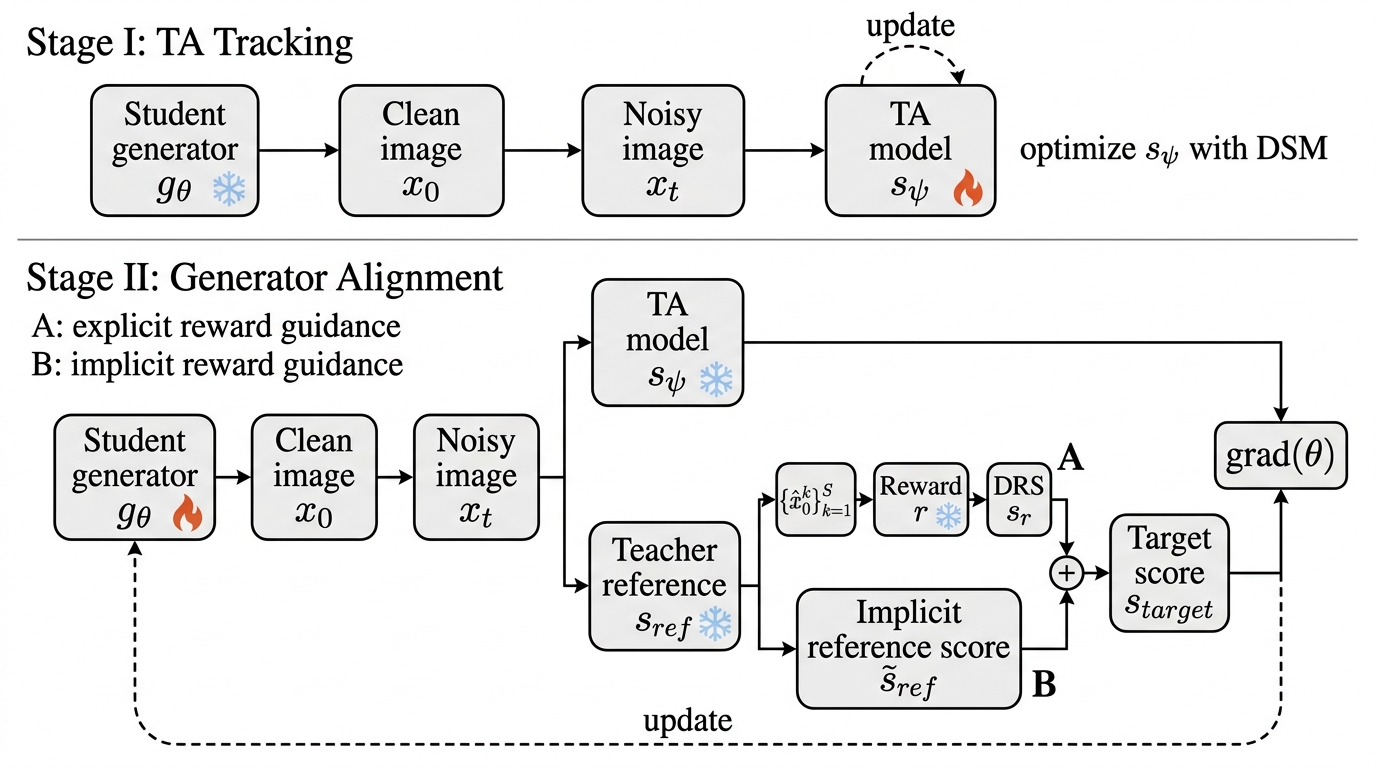}
\caption{\textbf{The \textsc{Didr} training framework.} \textbf{Stage I (TA):} $s_\psi$ tracks the generator marginals via DSM. \textbf{Stage II (Generator):} $g_\theta$ is updated toward the reward-tilted target score $\tilde{s}_{\mathrm{ref}} + s_r$.}
\label{fig:workflow}
\end{figure}

\begin{algorithm}[t]
\caption{\textsc{Didr} Training (see Appendix~\ref{app:alg_detail} for full pseudocode)}
\label{alg:didr}
\SetAlgoLined
\small
\While{\textnormal{not converged}}{
\textbf{Stage I -- TA update:} sample $x_0=g_\theta(z,c)$, diffuse to $x_t$, and update $s_\psi$ by DSM to track $p_{\theta,t}$\;
\BlankLine
\textbf{Stage II -- Generator update:}\;
\quad sample $x_0=g_\theta(z,c)$ and diffuse to $x_t$\;
\quad run $K$ differentiable $S$-step reference denoising chains from $x_t$ to obtain $\{\hat{x}_0^{(k)}\}_{k=1}^{K}$\;
\quad compute $s_r=\frac{1}{\tau}\sum_k\omega^{(k)}\nabla_{x_t}r(\hat{x}_0^{(k)},c)$, with $\omega^{(k)}\propto\exp(r(\hat{x}_0^{(k)},c)/\tau)$\;
\quad update $\theta$ using $w(t)(s_\psi-\tilde{s}_{\mathrm{ref}}-s_r)\partial x_t/\partial\theta$\;
}
\end{algorithm}

\section{Additional Qualitative Results}
\label{app:qualitative}

This section provides supplementary visual comparisons. Figure~\ref{fig:tau_ablation} illustrates the qualitative effect of the temperature hyperparameter $\tau$ on generated images. Figure~\ref{fig:onestep_vs_multistep} compares one-step \textsc{Didr} against 50-step multi-step baselines on matched prompts. Figure~\ref{fig:failure} documents representative failure cases.

\begin{figure}[h]
\centering
\includegraphics[width=\textwidth]{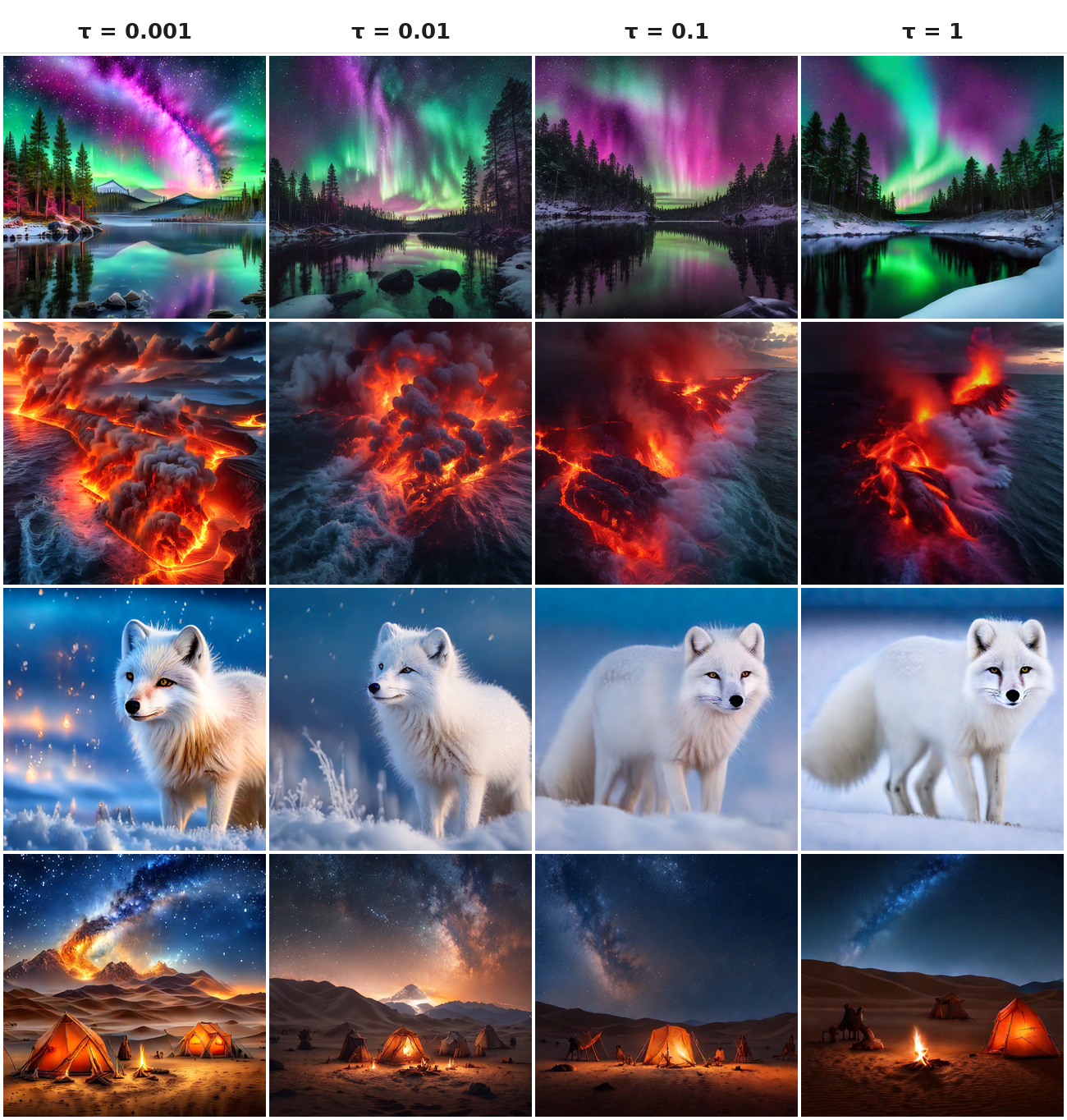}
\caption{\textbf{Qualitative effect of temperature $\tau$.} Images generated from the same prompt at decreasing $\tau$ (left to right). Smaller $\tau$ sharpens reward weighting and increases visual appeal, but introduces over-saturation and fine-detail artifacts at very low values, illustrating the preference--fidelity trade-off. Prompts in Appendix~\ref{app:prompts}.}
\label{fig:tau_ablation}
\end{figure}

\begin{figure}[h]
\centering
\includegraphics[width=\textwidth]{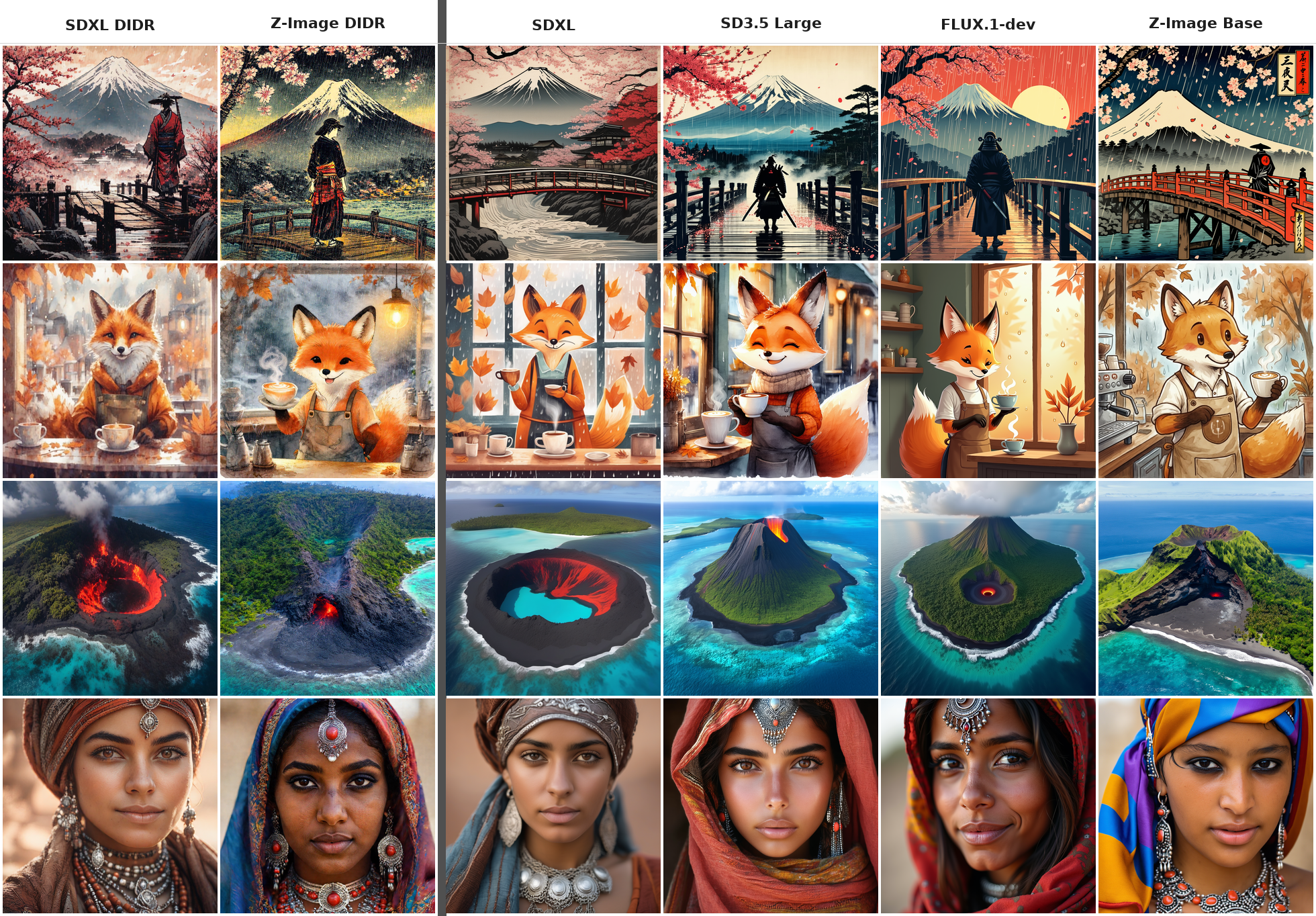}
\caption{\textbf{Qualitative comparison against multi-step baselines.} One-step \textsc{Didr} compared with 50-step SDXL, 50-step Z-Image, 50-step FLUX-dev, and 28-step SD3.5-Large on matched prompts. \textsc{Didr} achieves comparable or superior visual quality and prompt fidelity in a single inference step. Prompts in Appendix~\ref{app:prompts}.}
\label{fig:onestep_vs_multistep}
\end{figure}

\begin{figure}[h]
\centering
\includegraphics[width=\textwidth]{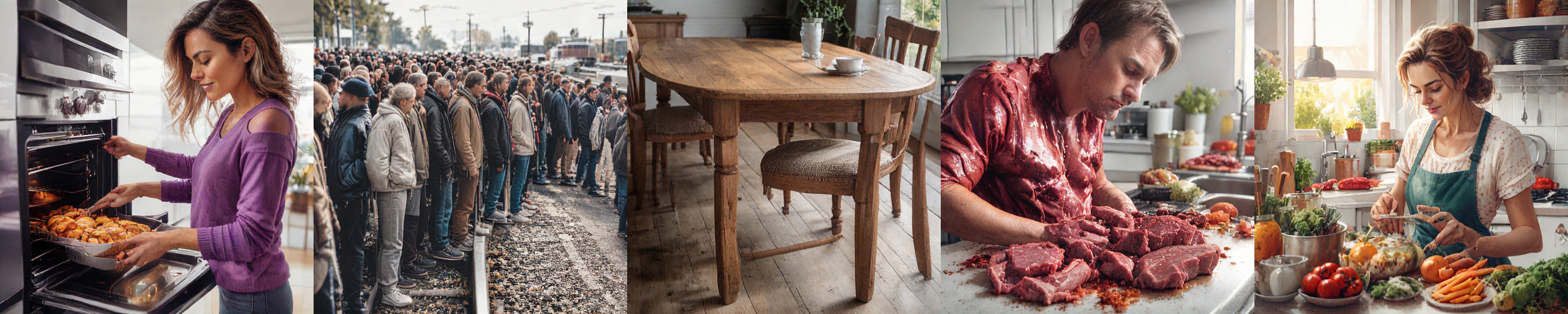}
\caption{\textbf{Failure cases of \textsc{Didr}:} anatomical errors, incorrect counts, structural distortions, and subject--background entanglement. Prompts in Appendix~\ref{app:prompts}.}
\label{fig:failure}
\end{figure}

\section{Theoretical Derivations}
\label{app:theory}

\subsection{Proof of the RLHF-Target Equivalence (Eq.~\eqref{eq:rlhf_kl})}
\label{app:proof_lem1}
\begin{proof}
Let $Z(c)\triangleq\int q_0(x_0|c)\exp(r(x_0,c)/\tau)\diff x_0<\infty$ and $q^*(x_0|c)\triangleq q_0(x_0|c)\exp(r/\tau)/Z(c)$. Expanding $\mathcal{L}(\theta)$:
\begin{align}
\mathcal{L}(\theta)
&= \tau\int p_\theta\!\left[\log p_\theta - \log q_0 - \frac{r}{\tau}\right]\diff x_0
= \tau\int p_\theta\!\left[\log p_\theta - \log\!\left(q_0\,e^{r/\tau}\right)\right]\diff x_0.
\end{align}
Since $q_0 e^{r/\tau} = Z(c)\,q^*$, this becomes $\tau\,\mathcal{D}_{\mathrm{KL}}(p_\theta\|q^*) - \tau\log Z(c)$. As $Z(c)$ is independent of $\theta$, the two objectives share the same minimizer.
\end{proof}

\subsection{Proof of Proposition~\ref{thm:ikl_equiv} (IKL Shares the RLHF Minimizer)}
\label{app:proof_thm_ikl_equiv}
\begin{proof}
\textbf{(i) Non-negativity.} Follows immediately from $\mathcal{D}_{\mathrm{KL}}\geq 0$ and $w(t)>0$.

\textbf{(ii) Convexity.} $\mathcal{D}_{\mathrm{KL}}(p\|q^*)$ is convex in $p$ for fixed $q^*$; a positively weighted integral of convex functionals is convex.

\textbf{(iii) Unique minimizer.} $\mathcal{L}_{\mathrm{IKL}}=0$ iff $\mathcal{D}_{\mathrm{KL}}(p_t\|q_t^*)=0$ for a.e.\ $t\in(0,T]$, since $w(t)>0$ and $\mathcal{D}_{\mathrm{KL}}=0$ iff the distributions coincide. This forces $p_t=q_t^*$ on a set of full measure in $(0,T]$, hence on a sequence $t_n\downarrow 0$. Since SDE marginals are weakly continuous in $t$ (standard for Lipschitz drift and diffusion coefficients), taking $n\to\infty$ gives $p_0=q_0^*=q^*$, which by Eq.~\eqref{eq:reward_tilted_target} is the unique minimizer of $\mathcal{L}$.
\end{proof}

\subsection{Proof of Theorem~\ref{thm:score_gradient} (Score-based DIDR Gradient)}
\label{app:proof_thm2}
\begin{proof}
Let $x_t = \mathcal{F}(g_\theta(z,c),\mathbf{w},t)$, $v_\theta\triangleq\partial x_t/\partial\theta$, and $s_\theta\triangleq\nabla_{x_t}\log p_{\theta,t}$. Differentiating $\mathcal{D}_{\mathrm{KL}}(p_{\theta,t}\|q_t^*)=\int p_{\theta,t}\log(p_{\theta,t}/q_t^*)\diff x_t$ in $\theta$:
\begin{equation}
\nabla_\theta\,\mathcal{D}_{\mathrm{KL}}(p_{\theta,t}\|q_t^*)
= \underbrace{\int\frac{\partial p_{\theta,t}}{\partial\theta}\log\frac{p_{\theta,t}}{q_t^*}\,\diff x_t}_{\mathrm{(I)}}
+ \underbrace{\int p_{\theta,t}\,\frac{\partial\log p_{\theta,t}}{\partial\theta}\,\diff x_t}_{\mathrm{(II)}}.
\end{equation}
Term (II) equals $\partial_\theta\!\int p_{\theta,t}\diff x_t = 0$.

\textbf{Continuity equation.} For any smooth test function $\phi$, differentiating both sides of $\E_{z,\mathbf{w}}[\phi(x_t)]=\int\phi\,p_{\theta,t}\diff x_t$ in $\theta$:
\begin{equation}
\int\phi\,\frac{\partial p_{\theta,t}}{\partial\theta}\diff x_t
= \E_{z,\mathbf{w}}\!\bigl[\nabla_{x_t}\phi(x_t)\cdot v_\theta\bigr]
= \int p_{\theta,t}\,\nabla_{x_t}\phi\cdot v_\theta\,\diff x_t
= -\int\phi\,\nabla_{x_t}\cdot(p_{\theta,t}\,v_\theta)\,\diff x_t,
\end{equation}
where the last step is integration by parts. Since $\phi$ is arbitrary,
\begin{equation}
\label{eq:continuity}
\frac{\partial p_{\theta,t}}{\partial\theta} = -\nabla_{x_t}\cdot(p_{\theta,t}\,v_\theta).
\end{equation}

\textbf{KL gradient.} Substituting \eqref{eq:continuity} into term (I) and integrating by parts:
\begin{align}
\mathrm{(I)}
&= -\int\nabla_{x_t}\cdot(p_{\theta,t}\,v_\theta)\,\log\frac{p_{\theta,t}}{q_t^*}\,\diff x_t \nonumber\\
&= \int p_{\theta,t}\,v_\theta\cdot\nabla_{x_t}\log\frac{p_{\theta,t}}{q_t^*}\,\diff x_t \nonumber\\
&= \int p_{\theta,t}\,v_\theta\cdot\bigl(s_\theta(x_t,t,c) - \nabla_{x_t}\log q_t^*(x_t|c)\bigr)\,\diff x_t.
\end{align}
Rewriting as an expectation over $(z,\mathbf{w})$ via $\int p_{\theta,t}\,\varphi\,\diff x_t = \E_{z,\mathbf{w}}[\varphi(x_t)]$ and substituting $v_\theta = \partial x_t/\partial\theta$:
\begin{equation}
\nabla_\theta\,\mathcal{D}_{\mathrm{KL}}(p_{\theta,t}\|q_t^*) = \E_{z,\mathbf{w}}\!\left[\bigl(s_\theta(x_t,t,c)-\nabla_{x_t}\log q_t^*(x_t|c)\bigr)\frac{\partial x_t}{\partial\theta}\right].
\end{equation}
Multiplying by $w(t)$, integrating over $t\in(0,T]$, and exchanging the $t$-integral with the expectation via Fubini gives Eq.~\eqref{eq:ikl_grad}.

\textbf{Target score decomposition} ($\nabla_{x_t}\log q_t^* = s_{\mathrm{ref}} + s_r$)\textbf{.}
Substituting $q^*(x_0|c) = q_0(x_0|c)\exp(r/\tau)/Z(c)$ into $q_t^*(x_t|c)=\int q_t(x_t|x_0)\,q^*(x_0|c)\diff x_0$ and using the Bayes factorization $q_0(x_0|c)\,q_t(x_t|x_0)=q_t(x_t|c)\,q(x_0|x_t,c)$:
\begin{align}
q_t^*(x_t|c)
&= \frac{1}{Z(c)}\int q_t(x_t|x_0)\,q_0(x_0|c)\exp\!\left(\frac{r(x_0,c)}{\tau}\right)\diff x_0 \nonumber\\
&= \frac{q_t(x_t|c)}{Z(c)}\int q(x_0|x_t,c)\exp\!\left(\frac{r(x_0,c)}{\tau}\right)\diff x_0 \nonumber\\
&= \frac{q_t(x_t|c)}{Z(c)}\;\E_{x_0\sim q(x_0|x_t,c)}\!\left[\exp\!\left(\frac{r(x_0,c)}{\tau}\right)\right].
\end{align}

Taking the logarithm of the factored form:
\begin{equation}
\log q_t^*(x_t|c) = \log q_t(x_t|c) + \log\E_{x_0\sim q(x_0|x_t,c)}\!\left[\exp\!\left(\frac{r(x_0,c)}{\tau}\right)\right] - \log Z(c).
\end{equation}
Applying $\nabla_{x_t}$: since $Z(c)$ does not depend on $x_t$, its gradient vanishes, giving
\begin{equation}
\nabla_{x_t}\log q_t^*(x_t|c)
= \underbrace{\nabla_{x_t}\log q_t(x_t|c)}_{=\,s_{\mathrm{ref}}(x_t,t,c)}
+ \underbrace{\nabla_{x_t}\log\E_{x_0\sim q(x_0|x_t,c)}\!\left[\exp\!\left(\frac{r(x_0,c)}{\tau}\right)\right]}_{\mathrm{DRS}(x_t,t,c)},
\end{equation}
which is exactly the decomposition in Eq.~\eqref{eq:ikl_grad}. The DRS term is intractable as written; Appendix~\ref{app:proof_lem4} derives a differentiable estimator for it.
\end{proof}

\subsection{Proof of Proposition~\ref{prop:drp} (Diffused Reward Proxy)}
\label{app:proof_lem4}
\begin{proof}
We derive a differentiable estimator for
\begin{equation}
\mathrm{DRS}(x_t,t,c)
= \nabla_{x_t}\log\E_{x_0\sim q(x_0|x_t,c)}\!\left[\exp\!\left(\frac{r(x_0,c)}{\tau}\right)\right].
\end{equation}

By the identity $\nabla\log f = \nabla f/f$:
\begin{equation}
\mathrm{DRS}(x_t,t,c)
= \frac{\nabla_{x_t}\,\E_{x_0\sim q(x_0|x_t,c)}\!\left[\exp\!\left(\tfrac{r(x_0,c)}{\tau}\right)\right]}
{\E_{\tilde x_0\sim q(\tilde x_0|x_t,c)}\!\left[\exp\!\left(\tfrac{r(\tilde x_0,c)}{\tau}\right)\right]}.
\end{equation}

To handle the numerator, we obtain differentiable approximate samples from the reference posterior $q(x_0|x_t,c)$ by running a denoising chain with the frozen reference model. Denote this map $x_0 = \mathcal{G}_{\mathrm{ref}}(x_t,\boldsymbol{\epsilon})$, initialized at $\hat{x}_{t_S}=x_t$, where $\boldsymbol{\epsilon}=(\epsilon^{(S)},\ldots,\epsilon^{(1)})$ is the sequence of injected noise vectors (empty for deterministic chains). The chain differs by model type:

\emph{VP diffusion (SDXL).} At each step $j=S,S{-}1,\ldots,1$, the Tweedie formula gives a clean-image estimate,
\[
\hat{x}_0^{(j)} = \frac{\hat{x}_{t_j}+\sigma_{t_j}^2\,s_{\mathrm{ref}}(\hat{x}_{t_j},t_j,c)}{\alpha_{t_j}},
\]
followed by a stochastic DDPM-style update
\[
\hat{x}_{t_{j-1}} = \alpha_{t_{j-1}}\hat{x}_0^{(j)} + \sigma_{t_{j-1}}\frac{\hat{x}_{t_j}-\alpha_{t_j}\hat{x}_0^{(j)}}{\sigma_{t_j}} + \tilde\sigma_{t_j}\,\epsilon^{(j)},
\qquad \epsilon^{(j)}\sim\mathcal{N}(0,I),
\]
where $\tilde\sigma_{t_j}=\sigma_{t_{j-1}}\sqrt{1-\dfrac{\alpha_{t_j}^2\,\sigma_{t_{j-1}}^2}{\alpha_{t_{j-1}}^2\,\sigma_{t_j}^2}}$ is the DDPM posterior standard deviation. Injecting independent noise $\epsilon^{(j)}$ for each of the $K$ chains yields $K$ diverse approximate posterior samples; the output of each chain is $\hat{x}_0^{(1)}$.

\emph{Flow matching (Z-Image).} A fixed Euler schedule yields identical $\hat{x}_0$ across all $K$ chains. Instead, each chain draws $S$ timesteps uniformly at random from $(0,t_S]$, runs the deterministic Euler ODE on its own schedule,
\[
\hat{x}_{t_{j-1}^{(k)}} = \hat{x}_{t_j^{(k)}} + \bigl(t_{j-1}^{(k)}-t_j^{(k)}\bigr)\,v_{\mathrm{ref}}\!\left(\hat{x}_{t_j^{(k)}},t_j^{(k)},c\right),
\]
and returns $\hat{x}_0^{(1)}=\hat{x}_{t_1^{(k)}}-t_1^{(k)}\,v_{\mathrm{ref}}(\hat{x}_{t_1^{(k)}},t_1^{(k)},c)$, yielding $K$ diverse approximate denoising endpoints induced by randomized discretization paths.

In both cases $\mathcal{G}_{\mathrm{ref}}$ is differentiable in $x_t$ for fixed $\boldsymbol{\epsilon}$, since all operations are compositions of smooth network evaluations and linear maps. Under this reparameterization the expectation becomes $\E_{\boldsymbol{\epsilon}}[\exp(r(\mathcal{G}_{\mathrm{ref}}(x_t,\boldsymbol{\epsilon}),c)/\tau)]$, in which $x_t$ enters explicitly through $\mathcal{G}_{\mathrm{ref}}$ rather than through the measure. Interchanging gradient and expectation via dominated convergence (justified when $r$ is smooth with bounded gradient and $\mathcal{G}_{\mathrm{ref}}$ is uniformly Lipschitz in $x_t$):
\begin{align}
\nabla_{x_t}\,\E_{x_0\sim q(x_0|x_t,c)}\!\left[\exp\!\left(\frac{r(x_0,c)}{\tau}\right)\right]
&= \E_{\boldsymbol{\epsilon}}\!\left[\nabla_{x_t}\exp\!\left(\frac{r(\mathcal{G}_{\mathrm{ref}}(x_t,\boldsymbol{\epsilon}),c)}{\tau}\right)\right].
\end{align}
Applying the chain rule:
\begin{equation}
\nabla_{x_t}\exp\!\left(\frac{r(\mathcal{G}_{\mathrm{ref}}(x_t,\boldsymbol{\epsilon}),c)}{\tau}\right)
= \exp\!\left(\frac{r(x_0,c)}{\tau}\right)\cdot\frac{1}{\tau}\underbrace{\frac{\partial r(x_0,c)}{\partial x_0}\frac{\partial\mathcal{G}_{\mathrm{ref}}(x_t,\boldsymbol{\epsilon})}{\partial x_t}}_{\triangleq\;\nabla_{x_t}r(x_0,c)},
\end{equation}
where $\nabla_{x_t}r(x_0,c)$ denotes the pathwise gradient of $r$ through the differentiable denoising chain. Converting back to the $q(x_0|x_t,c)$ expectation:
\begin{equation}
\nabla_{x_t}\,\E_{x_0\sim q(x_0|x_t,c)}\!\left[\exp\!\left(\frac{r(x_0,c)}{\tau}\right)\right]
= \E_{x_0\sim q(x_0|x_t,c)}\!\left[\exp\!\left(\frac{r(x_0,c)}{\tau}\right)\cdot\frac{1}{\tau}\nabla_{x_t}r(x_0,c)\right].
\end{equation}

Dividing numerator and denominator:
\begin{equation}
\mathrm{DRS}(x_t,t,c)
= \E_{x_0\sim q(x_0|x_t,c)}\!\left[
\frac{\exp\!\left(\tfrac{r(x_0,c)}{\tau}\right)}
{\E_{\tilde x_0\sim q(\tilde x_0|x_t,c)}\!\left[\exp\!\left(\tfrac{r(\tilde x_0,c)}{\tau}\right)\right]}
\cdot\frac{1}{\tau}\nabla_{x_t}r(x_0,c)
\right],
\end{equation}
which is precisely the softmax-weighted gradient estimator in Eq.~\eqref{eq:drp_estimator}.
\end{proof}

\begin{remark}[Regularity assumptions]
The interchange of gradient and expectation above requires: (i) $r$ is smooth with bounded gradient, and (ii) $\mathcal{G}_{\mathrm{ref}}$ is uniformly Lipschitz in $x_t$. Condition~(i) holds for standard differentiable reward models such as PickScore. Condition~(ii) holds because $\mathcal{G}_{\mathrm{ref}}$ is a finite composition of smooth network evaluations and linear maps, hence globally Lipschitz on any compact domain.
\end{remark}

\subsection{A Bimodal Gaussian Example of Terminal Reward Domination}
\label{app:reward_domination_gaussian}

All results below hold under the mild assumption that the two modes are well-separated ($\mu/\sigma\gg 1$), in which regime $\Phi(-\mu/\sigma)\approx 0$ and every approximation becomes exact as $\mu/\sigma\to\infty$. The $\approx$ symbol denotes equality up to corrections of order $\Phi(-\mu/\sigma)$.

\paragraph{Setup.}
Consider the bimodal reference
\[
p_0
=
\tfrac12\mathcal N(-\mu,\sigma^2)
+
\tfrac12\mathcal N(\mu,\sigma^2),
\qquad
\mu>\sigma>0,
\]
the generator family
\[
q_{\alpha,0}
=
(1-\alpha)\mathcal N(-\mu,\sigma^2)
+
\alpha\mathcal N(\mu,\sigma^2),
\qquad
\alpha\in[0,1],
\]
and the binary reward $r(x_0)=\mathbf 1[x_0>0]$.

\paragraph{Exact reward expectation.}
By direct computation,
\begin{align}
\E_{q_{\alpha,0}}[r]
&= (1-\alpha)\,P\!\bigl(\mathcal N(-\mu,\sigma^2)>0\bigr)
+\alpha\,P\!\bigl(\mathcal N(\mu,\sigma^2)>0\bigr) \nonumber\\
&= (1-\alpha)\,\Phi\!\left(-\tfrac{\mu}{\sigma}\right)
+\alpha\,\Phi\!\left(\tfrac{\mu}{\sigma}\right)
= \alpha + (1-2\alpha)\,\Phi\!\left(-\tfrac{\mu}{\sigma}\right),
\end{align}
where $\Phi$ is the standard normal CDF. Since $\Phi(-\mu/\sigma)\to 0$ as $\mu/\sigma\to\infty$,
\[
\E_{q_{\alpha,0}}[r] \approx \alpha,
\qquad
\frac{d}{d\alpha}\E_{q_{\alpha,0}}[r]\big|_{\alpha=1} = 1 - 2\Phi\!\left(-\tfrac{\mu}{\sigma}\right) \approx 1.
\]

\paragraph{Forward marginals.}
Under the VP forward process with $\bar\alpha_t=e^{-\gamma t}$, $t\in[0,\infty)$, the forward kernel is $q_t(x_t|x_0)=\mathcal{N}(x_t;\sqrt{\bar\alpha_t}\,x_0,(1-\bar\alpha_t)I)$, so the marginals are
\[
q_{\alpha,t}
=
(1-\alpha)\phi_{-,t}
+
\alpha\phi_{+,t},
\qquad
p_t
=
\tfrac12\phi_{-,t}
+
\tfrac12\phi_{+,t},
\]
where $\phi_{\pm,t}=\mathcal N(\pm m_t,\Sigma_t)$ with
\[
m_t=\sqrt{\bar\alpha_t}\,\mu,
\qquad
\Sigma_t=1-\bar\alpha_t(1-\sigma^2).
\]

\paragraph{Objective and convexity.}
The endpoint-reward objective (to be \emph{minimized}) is
\[
\mathcal{L}_{\rm term}(\alpha)
=
-\E_{q_{\alpha,0}}[r]
+
\tau\int_0^\infty
D_t(\alpha)\,\diff t,
\qquad
D_t(\alpha):=
\mathrm{KL}(q_{\alpha,t}\|p_t).
\]
Since $q_{\alpha,t}$ is affine in $\alpha$ and KL divergence is convex in its first argument, $D_t(\alpha)$ is convex in $\alpha$, and $-\E_{q_{\alpha,0}}[r]$ is linear in $\alpha$. Hence $\mathcal{L}_{\rm term}$ is convex in $\alpha$, so $\mathcal{L}_{\rm term}'$ is non-decreasing on $[0,1]$. Consequently, $\alpha^*=1$ (collapse to the positive mode) occurs whenever the left derivative satisfies
\[
\mathcal{L}_{\rm term}'(1^-)\le 0,
\]
since convexity then forces $\mathcal{L}_{\rm term}'(\alpha)\le\mathcal{L}_{\rm term}'(1^-)\le 0$ for all $\alpha\in[0,1)$, making $\mathcal{L}_{\rm term}$ non-increasing throughout and the minimum attained at $\alpha^*=1$.

\paragraph{Computing $D_t'(1)$.}
Differentiating $D_t(\alpha)=\mathrm{KL}(q_{\alpha,t}\|p_t)$ with respect to $\alpha$ and using $\partial_\alpha q_{\alpha,t}=\phi_{+,t}-\phi_{-,t}$ (together with the fact that $\partial_\alpha\int q_{\alpha,t}\,\diff x=0$, so the term from differentiating the $\log q_{\alpha,t}$ factor vanishes by normalization), we obtain
\[
D_t'(\alpha)
=
\int
(\phi_{+,t}-\phi_{-,t})\,
\log\frac{q_{\alpha,t}(x)}{p_t(x)}
\,\diff x.
\]
At $\alpha=1$, $q_{1,t}=\phi_{+,t}$, so
\[
D_t'(1)
=
\int
(\phi_{+,t}(x)-\phi_{-,t}(x))\,
\log\frac{\phi_{+,t}(x)}{\frac12(\phi_{+,t}(x)+\phi_{-,t}(x))}
\,\diff x.
\]
Since $\phi_{\pm,t}=\mathcal{N}(\pm m_t,\Sigma_t)$, the Gaussian log-ratio is
\[
\log\frac{\phi_{+,t}(x)}{\phi_{-,t}(x)}
=
-\frac{(x-m_t)^2-(x+m_t)^2}{2\Sigma_t}
=
\frac{2m_t x}{\Sigma_t},
\]
so $\phi_{-,t}(x)/\phi_{+,t}(x)=\exp(-2m_t x/\Sigma_t)$. Substituting:
\[
\log\frac{\phi_{+,t}(x)}{\frac12(\phi_{+,t}+\phi_{-,t})}
=
\log\frac{2}{1+e^{-2m_tx/\Sigma_t}}.
\]

We now simplify by symmetry. Split the integral into contributions from $\phi_{+,t}$ and $\phi_{-,t}$:
\[
D_t'(1)
=
\int\phi_{+,t}(x)\,\log\frac{2}{1+e^{-2m_tx/\Sigma_t}}\,\diff x
-
\int\phi_{-,t}(x)\,\log\frac{2}{1+e^{-2m_tx/\Sigma_t}}\,\diff x.
\]
In the second integral, substitute $x\mapsto -x$; since $\phi_{-,t}(-x)=\phi_{+,t}(x)$:
\[
\int\phi_{-,t}(x)\,\log\frac{2}{1+e^{-2m_tx/\Sigma_t}}\,\diff x
=
\int\phi_{+,t}(x)\,\log\frac{2}{1+e^{2m_tx/\Sigma_t}}\,\diff x.
\]
Combining both integrals:
\begin{equation}
\label{eq:Dt_prime_1}
D_t'(1)
=
\int\phi_{+,t}(x)\,\log\frac{1+e^{2m_tx/\Sigma_t}}{1+e^{-2m_tx/\Sigma_t}}\,\diff x
=
\E_{x\sim\phi_{+,t}}\!\left[\log\frac{1+e^{2m_tx/\Sigma_t}}{1+e^{-2m_tx/\Sigma_t}}\right].
\end{equation}

Finally, we invoke the well-separated assumption. Define $a_t\triangleq m_t/\sqrt{\Sigma_t}$. Under $\phi_{+,t}=\mathcal{N}(m_t,\Sigma_t)$, a typical sample satisfies $x\approx m_t$, so $2m_t x/\Sigma_t\approx 2m_t^2/\Sigma_t=2a_t^2$. In the regime $\mu/\sigma\gg1$ (so $a_t\gg1$ for all $t$ where the diffusion has not yet erased the modes), $e^{2m_tx/\Sigma_t}\gg1$ and $e^{-2m_tx/\Sigma_t}\approx0$ throughout the support of $\phi_{+,t}$. Therefore,
\[
\log\frac{1+e^{2m_tx/\Sigma_t}}{1+e^{-2m_tx/\Sigma_t}}
\approx
\log e^{2m_tx/\Sigma_t}
=
\frac{2m_t x}{\Sigma_t}.
\]
Substituting into Eq.~\eqref{eq:Dt_prime_1} and using $\E_{x\sim\phi_{+,t}}[x]=m_t$:
\[
D_t'(1)
\approx
\E_{x\sim\phi_{+,t}}\!\left[\frac{2m_t x}{\Sigma_t}\right]
=
\frac{2m_t}{\Sigma_t}\,\E_{x\sim\phi_{+,t}}[x]
=
\frac{2m_t^2}{\Sigma_t}.
\]

\paragraph{Collapse condition.}
Combining the derivative of the reward term ($\approx -1$ after negation) and the regularizer:
\[
\mathcal{L}_{\rm term}'(1)
\approx
-1
+
\tau\int_0^\infty
\frac{2m_t^2}{\Sigma_t}\,\diff t.
\]
The condition $\mathcal{L}_{\rm term}'(1)\le0$ gives the collapse threshold
\[
\tau
\le
\frac{1}{B_{\rm crit}},
\qquad
B_{\rm crit}
\triangleq
\int_0^\infty
\frac{2m_t^2}{\Sigma_t}\,\diff t.
\]

\paragraph{Closed-form evaluation of $B_{\rm crit}$.}
Substituting $m_t^2=\bar\alpha_t\mu^2$ and $\Sigma_t=1-\bar\alpha_t(1-\sigma^2)$:
\[
B_{\rm crit}
=
\int_0^\infty
\frac{2\mu^2\,e^{-\gamma t}}
{1-e^{-\gamma t}(1-\sigma^2)}
\,\diff t.
\]
Change variables $u=e^{-\gamma t}$, so $\diff u=-\gamma u\,\diff t$ and the limits $t:0\to\infty$ become $u:1\to 0$:
\[
B_{\rm crit}
=
\int_1^0
\frac{2\mu^2\,u}
{1-(1-\sigma^2)u}
\cdot\frac{-\diff u}{\gamma u}
=
\frac{2\mu^2}{\gamma}
\int_0^1
\frac{\diff u}{1-(1-\sigma^2)u}.
\]
For $\sigma^2\neq1$, the integral evaluates via $\int_0^1\frac{\diff u}{1-cu}=\frac{-\log(1-c)}{c}$ with $c=1-\sigma^2$:
\[
B_{\rm crit}
=
\frac{2\mu^2}{\gamma}\cdot\frac{-\log\sigma^2}{1-\sigma^2}.
\]
Therefore, the collapse condition is
\[
\boxed{
\tau
\le
\frac{\gamma(1-\sigma^2)}{2\mu^2(-\log\sigma^2)}
\quad
\Longrightarrow
\quad
\alpha^*=1.
}
\]
The case $\sigma^2=1$ follows by L'H\^opital's rule (or a direct integral): $\int_0^1 \diff u/(1-0)=1$, giving
\[
\boxed{
\tau\le \frac{\gamma}{2\mu^2}
\quad
\Longrightarrow
\quad
\alpha^*=1.
}
\]

This shows that terminal reward domination occurs under the integral forward KL penalty for any finite $\tau$ below a closed-form threshold. Larger $\gamma$ (faster diffusion) raises the critical $\tau$, making collapse to the rewarded mode easier to trigger at any fixed regularization strength.

\section{1-D Toy Experiment}
\label{app:toy_experiment}

This section describes the setup for the 1-D empirical validation of terminal reward domination (Figure~\ref{fig:schematic_combined}(b)).

\paragraph{Reference distribution and reward.}
The reference is a symmetric bimodal Gaussian
\[
q_0 = \tfrac{1}{2}\mathcal{N}(-\mu,\sigma^2) + \tfrac{1}{2}\mathcal{N}(\mu,\sigma^2), \qquad \mu=2,\;\sigma=0.5,
\]
with ideal binary reward $r_{\mathrm{hard}}(x)=\mathbf{1}[x>0]$. For gradient-based training, we use the smooth approximation
\[
r(x)=\operatorname{sigmoid}(\beta x), \qquad \beta=20,
\]
so that $r(x)$ approximates $r_{\mathrm{hard}}(x)$ except in a narrow neighborhood of the decision boundary while remaining differentiable.

\paragraph{Forward process.}
VP diffusion with rate $\gamma=20$: $\bar\alpha_t=e^{-\gamma t}$, $t\in[0,T]$, $T=0.25$, giving $\bar\alpha_T=e^{-5}\approx0.007$. This places the experiment firmly in the collapse regime: the closed-form domination threshold derived in Appendix~\ref{app:reward_domination_gaussian} evaluates to $\tau_{\mathrm{crit}}\approx1.35>\tau=1$.

\paragraph{Networks.}
Both $s_{\mathrm{ref}}$ and $s_\psi$ are MLPs with input $(x,t)\in\mathbb{R}^2$, three hidden layers of width~128, SiLU activations, and a scalar output (noise-prediction parameterisation; converted to score via $s=-\hat{\varepsilon}/\sqrt{1-\bar\alpha_t}$).
The generator $g_\theta:\mathcal{N}(0,1)\to\mathbb{R}$ uses the same architecture with a scalar input $z$.

\paragraph{Training pipeline.}
\begin{enumerate}[1.]
    \item \textbf{Reference model.} $s_{\mathrm{ref}}$ is trained for 10{,}000 steps via DSM on samples from $q_0$ (Adam, lr=$3\times10^{-4}$, batch~2{,}048). Weights are frozen afterwards.
    \item \textbf{Generator distillation.} $g_\theta$ is initialised by regressing against 30-step DDIM samples from $s_{\mathrm{ref}}$ using matched noise (3{,}000 steps, Adam, lr=$10^{-3}$, batch~2{,}048).
    \item \textbf{Teaching Assistant initialisation.} $s_\psi$ is copied from $s_{\mathrm{ref}}$ and fine-tuned during alignment.
    \item \textbf{Alignment.} 6{,}000 outer steps; each step runs 5 TA DSM updates (Adam, lr=$3\times10^{-4}$) followed by one generator update (Adam, lr=$10^{-4}$, batch~2{,}048). The two methods differ only in $s_{\mathrm{target}}$:
\end{enumerate}

\begin{center}
\begin{tabular}{ll}
\toprule
Method & Reward location \\
\midrule
DI++ & endpoint $x_0$ only \\
\textsc{Didr} & every noise level via DRS \\
\bottomrule
\end{tabular}
\end{center}

\paragraph{Evaluation.}
Final generators are evaluated on $n=10{,}000$ samples using the hard decision boundary $x>0$. For the ideal binary reward $r_{\mathrm{hard}}$, the theoretical optimal weight under $q^*$ is $P_{q^*}(x>0)=\operatorname{sigmoid}(1/\tau)\approx0.731$.
DI++ collapses to $P(x>0)\approx1$ while \textsc{Didr} converges to $P(x>0)\approx0.73$, confirming the theory.

\section{Evaluation Metrics}
\label{app:evaluation_metrics}

We use automatic metrics that measure complementary aspects of text-to-image generation. Preference metrics estimate human visual preference or aesthetics; text-alignment metrics measure whether the image satisfies the prompt; FID measures distributional fidelity against real images. For all metrics except FID, higher is better.

\paragraph{Preference metrics.}
PickScore~\citep{pickscore} is a CLIP-based human preference model trained on Pick-a-Pic pairwise user preferences, and is also the reward used for \textsc{Didr} training. ImageReward~\citep{xu2023imagereward} is a learned image-text reward model trained from human preference annotations for text-to-image outputs. HPSv2.1~\citep{hpsv2} evaluates human preference using the Human Preference Score benchmark and reports results over its standard prompt categories. Aesthetic Score~\citep{laionaes} uses the LAION aesthetic predictor, which estimates visual appeal from CLIP image embeddings and does not directly measure prompt faithfulness.

\paragraph{Text-alignment metrics.}
CLIPScore \citep{hessel2021clipscore,radford2021learning} measures image--text compatibility using CLIP embedding similarity. DPG-Bench \citep{dpgbench} evaluates semantic adherence on dense prompts containing multiple attributes and relations. GenEval \citep{geneval} is an object-focused benchmark that checks compositional correctness, including object presence, counting, color binding, position, and attribute relations.

\paragraph{Fidelity metric.}
FID~\citep{heusel2017gans} compares generated and real image distributions in Inception feature space:
\[
\mathrm{FID}
= \|\mu_r-\mu_g\|_2^2
+ \operatorname{Tr}\!\left(\Sigma_r+\Sigma_g-2(\Sigma_r\Sigma_g)^{1/2}\right),
\]
where $(\mu_r,\Sigma_r)$ and $(\mu_g,\Sigma_g)$ are Gaussian approximations of real and generated feature distributions. Lower FID indicates closer distributional match to the reference images, but it does not directly measure prompt alignment or human preference.

\section{Per-Category Results}
\label{app:geneval}

Tables~\ref{tab:geneval} and~\ref{tab:hpsv2} provide per-category breakdowns of GenEval and HPSv2.1 scores respectively. On GenEval, \textsc{Didr} achieves the highest overall score among one-step SDXL methods, with particularly strong gains on counting ($0.513$ vs.\ $0.428$ for DI*) and color attribute binding ($0.245$ vs.\ $0.215$). On HPSv2.1, \textsc{Didr}$_{\text{longer}}$ leads across all four style categories. For the Z-Image backbone, \textsc{zimage-Didr} improves substantially over the 1-step Z-Image-Turbo initialization (row marked $^\dagger$) on all categories, suggesting that the gains come from the alignment procedure rather than initialization quality.

\begin{table*}[!h]
\centering
\caption{GenEval per-category breakdown. \textbf{Bold}: best result within each one-step group. $^\dagger$: Z-Image-Turbo at 1~NFE (unofficial; \textsc{Didr} training initialization only). $-$: not reported.}
\label{tab:geneval}
\resizebox{\textwidth}{!}{\begin{tabular}{lclcccccccc}
\toprule
& & & & \multicolumn{6}{c}{\textbf{GenEval Category} $\uparrow$} & \\
\cmidrule(lr){5-10}
\textbf{Model} & \textbf{Steps} & \textbf{Arch.} & \textbf{Params} & \textbf{Single Obj.} & \textbf{Two Obj.} & \textbf{Counting} & \textbf{Colors} & \textbf{Color Attr.} & \textbf{Position} & \textbf{Overall} $\uparrow$ \\
\midrule
\multicolumn{11}{l}{\textit{Multi-step reference}} \\
SDXL~\citep{podell2023sdxl}           & 50 & UNet & 2.6B & 0.973 & 0.745 & 0.381 & 0.847 & 0.218 & 0.142 & 0.551 \\
SDXL-DPO~\citep{wallace2024diffusion} & 50 & UNet & 2.6B & 0.971 & 0.751 & 0.374 & 0.843 & 0.221 & 0.134 & 0.549 \\
SD3.5-large~\citep{sd35}              & 28 & DiT  & 8B   & 0.997 & 0.921 & 0.668 & 0.853 & 0.593 & 0.270 & 0.717 \\
FLUX-dev~\citep{fluxdev}              & 50 & DiT  & 12B  & 0.991 & 0.863 & 0.683 & 0.771 & 0.489 & 0.235 & 0.672 \\
Z-Image~\citep{zimage}                & 50 & DiT  & 6B   & 1.000 & 0.943 & 0.781 & 0.928 & 0.624 & 0.764 & 0.840 \\
\midrule
\multicolumn{11}{l}{\textit{One-step SDXL}} \\
SDXL DMD2~\citep{yin2024improved}                & 1 & UNet & 2.6B & 0.984 & 0.677 & 0.472 & 0.878 & 0.220 & \textbf{0.110} & 0.557 \\
SDXL Diff-Instruct~\citep{Luo2023DiffInstructAU} & 1 & UNet & 2.6B & 0.994 & 0.748 & 0.456 & 0.883 & 0.188 & 0.090 & 0.560 \\
SDXL Diff-Instruct++~\citep{anonymous2024diffinstruct} & 1 & UNet & 2.6B & 0.988 & 0.735 & 0.475 & 0.872 & 0.200 & 0.093 & 0.560 \\
SDXL Diff-Instruct*~\citep{luo2024onestep}       & 1 & UNet & 2.6B & 0.994 & 0.740 & 0.428 & 0.851 & 0.215 & 0.105 & 0.555 \\
\textbf{SDXL \textsc{Didr} (Ours)}                      & 1 & UNet & 2.6B & 0.991 & 0.727 & \textbf{0.513} & \textbf{0.891} & \textbf{0.245} & 0.108 & \textbf{0.579} \\
\textbf{SDXL \textsc{Didr}$_{\text{longer}}$ (Ours)}    & 1 & UNet & 2.6B & \textbf{1.000} & \textbf{0.758} & 0.472 & 0.883 & 0.190 & 0.095 & 0.566 \\
\midrule
\multicolumn{11}{l}{\textit{Z-Image backbone}} \\
Z-Image-Turbo~\citep{zimage}                     & 8 & DiT & 6B & \textbf{1.000} & \textbf{0.904} & \textbf{0.728} & \textbf{0.835} & \textbf{0.578} & \textbf{0.455} & \textbf{0.750} \\
\textbf{Z-Image \textsc{Didr} (Ours)}            & 1 & DiT & 6B & 0.997 & 0.846 & 0.625 & 0.811 & 0.533 & 0.340 & 0.692 \\
\hdashline
Z-Image-Turbo$^\dagger$~\citep{zimage}           & 1 & DiT & 6B & 0.972 & 0.513 & 0.559 & 0.795 & 0.300 & 0.293 & 0.572 \\
\bottomrule
\end{tabular}}
\end{table*}

\subsection*{HPSv2.1 Per-Category Results}
\label{app:hpsv2}

\begin{table*}[!t]
\centering
\caption{HPSv2.1 per-category breakdown. Same conventions as Table~\ref{tab:main}. $^\dagger$: Z-Image-Turbo at 1~NFE (unofficial; \textsc{Didr} training initialization only).}
\label{tab:hpsv2}
\resizebox{\textwidth}{!}{\begin{tabular}{lclcccccc}
\toprule
& & & & \multicolumn{5}{c}{\textbf{HPSv2.1} $\uparrow$} \\
\cmidrule(lr){5-9}
\textbf{Model} & \textbf{Steps} & \textbf{Arch.} & \textbf{Params} & \textbf{Animation} & \textbf{Concept-Art} & \textbf{Painting} & \textbf{Photo} & \textbf{Average} \\
\midrule
\multicolumn{9}{l}{\textit{Multi-step reference}} \\
SDXL~\citep{podell2023sdxl}            & 50 & UNet & 2.6B & 30.89 & 29.16 & 28.91 & 26.97 & 28.98 \\
SDXL-DPO~\citep{wallace2024diffusion}  & 50 & UNet & 2.6B & 32.08 & 30.80 & 30.82 & 28.12 & 30.45 \\
SD3.5-large~\citep{sd35}               & 28 & DiT  & 8B   & 31.87 & 30.05 & 30.31 & 28.01 & 30.06 \\
FLUX-dev~\citep{fluxdev}               & 50 & DiT  & 12B  & 31.95 & 30.35 & 31.21 & 29.10 & 30.65 \\
Z-Image~\citep{zimage}                 & 50 & DiT  & 6B   & 32.53    & 30.42    & 29.58    & 29.76    & 30.57 \\
\midrule
\multicolumn{9}{l}{\textit{One-step SDXL}} \\
SDXL DMD2~\citep{yin2024improved}                  & 1 & UNet & 2.6B & 31.53 & 29.86 & 29.43 & 29.14 & 29.99 \\
SDXL Diff-Instruct~\citep{Luo2023DiffInstructAU}   & 1 & UNet & 2.6B & 32.78 & 31.39 & 31.29 & 30.59 & 31.51 \\
SDXL Diff-Instruct++~\citep{anonymous2024diffinstruct} & 1 & UNet & 2.6B & 32.74 & 31.42 & 31.20 & 30.41 & 31.44 \\
SDXL Diff-Instruct*~\citep{luo2024onestep}          & 1 & UNet & 2.6B & 35.10 & 33.63 & 33.14 & 31.28 & 33.29 \\
\textbf{SDXL \textsc{Didr} (Ours)}                  & 1 & UNet & 2.6B & 35.64 & 34.21 & 33.81 & 31.41 & 33.77 \\
\textbf{SDXL \textsc{Didr}$_{\text{longer}}$ (Ours)} & 1 & UNet & 2.6B & \textbf{35.76} & \textbf{34.43} & \textbf{33.83} & \textbf{31.53} & \textbf{33.89} \\
\midrule
\multicolumn{9}{l}{\textit{Z-Image backbone}} \\
Z-Image-Turbo~\citep{zimage}                        & 8 & DiT & 6B & \textbf{33.44} & \textbf{31.99} & \textbf{31.26} & \textbf{30.76} & \textbf{31.86} \\
\textbf{Z-Image \textsc{Didr} (Ours)}               & 1 & DiT & 6B & 33.39 & 31.62 & 30.38 & 30.21 & 31.40 \\
\hdashline
Z-Image-Turbo$^\dagger$~\citep{zimage}              & 1 & DiT & 6B & 26.50 & 24.22 & 23.63 & 23.81 & 24.54 \\
\bottomrule
\end{tabular}}
\end{table*}

\section{Full Training Algorithm}
\label{app:alg_detail}

Algorithm~\ref{alg:didr_detail} provides the complete pseudocode for \textsc{Didr}, expanding the concise version in Algorithm~\ref{alg:didr}. The outer loop alternates between two stages. Stage~I updates the Teaching Assistant $s_\psi$ via denoising score matching to track the generator's current marginals. Stage~II updates the generator $g_\theta$: for each training step it samples a prompt and noise, forward-diffuses the generated image to a random time $t$, draws $K$ posterior samples via $S$-step reference denoising, computes the DRP estimate $s_r$ as a softmax-weighted gradient, and applies the IKL gradient. The CFG-corrected reference score $\tilde{s}_{\mathrm{ref}}$ is used throughout to maintain text alignment.

\begin{algorithm}[H]
\caption{\textsc{Didr}: Diff-Instruct with Diffused Reward (full pseudocode)}
\label{alg:didr_detail}
\SetAlgoLined
\SetKwInOut{Input}{Input}
\small
\Input{
prompt dataset $\mathcal{C}$; generator $g_\theta$; prior $p_z$;
reward model $r(\cdot, c)$; CFG scale $\cfg$;
reference diffusion $s_{\mathrm{ref}}$; Teaching Assistant (TA) $s_\psi$;
posterior samples $K$; temperature $\tau$; denoising steps $S$;
TA update rounds $K_{\mathrm{TA}}$; weights $w(t), \lambda(t)$.
}

\While{\textnormal{not converged}}{
\textbf{Stage I: Update Teaching Assistant (TA)}\;
\For{$k=1$ \KwTo $K_{\mathrm{TA}}$}{
Sample $c \sim \mathcal{C}, z \sim p_z, t \sim \pi(t)$\;
Generate $x_0 = g_\theta(z, c)$, diffuse $x_t \sim p_t(x_t|x_0)$\;
Update $\psi$ by minimizing DSM loss: $\mathcal{L}(\psi) = \lambda(t)\,\| s_\psi(x_t, t, c) - \nabla_{x_t} \log p_t(x_t|x_0) \|_2^2$\;
}

\BlankLine
\textbf{Stage II: Update Generator}\;
Sample $c \sim \mathcal{C}, z \sim p_z, t \sim \pi(t)$\;
Generate $x_0 = g_\theta(z, c)$, diffuse $x_t \sim p_t(x_t|x_0)$\;
Compute CFG score: $\tilde{s}_{\mathrm{ref}} = s_{\mathrm{ref}}(x_t,t,\emptyset) + \cfg\left[s_{\mathrm{ref}}(x_t,t,c) - s_{\mathrm{ref}}(x_t,t,\emptyset)\right]$\;
\BlankLine
\textit{Compute DRP estimate $s_r$ via posterior sampling:}\;
\For{$k=1$ \KwTo $K$}{
Run $S$-step differentiable denoising from $x_t$: $\hat{x}_0^{(k)} \leftarrow \mathrm{Denoise}_S(x_t, s_{\mathrm{ref}}, c)$\;
Compute $r^{(k)} = r(\hat{x}_0^{(k)}, c)$ and pathwise gradient $\nabla_{x_t} r^{(k)}$\;
}
Compute softmax weights: $\omega^{(k)} = \dfrac{\exp(r^{(k)}/\tau)}{\sum_{k'=1}^{K}\exp(r^{(k')}/\tau)}$\;
Compute $s_r = \dfrac{1}{\tau}\displaystyle\sum_{k=1}^{K} \omega^{(k)}\, \nabla_{x_t} r^{(k)}$\;
\BlankLine
Update $\theta$ via Eq.~\eqref{eq:grad_practical}: $\operatorname{Grad}(\theta) = w(t)\,\bigl(s_\psi(x_t,t,c) - (\tilde{s}_{\mathrm{ref}} + s_r)\bigr)\tfrac{\partial x_t}{\partial\theta}$\;
}
\Return{$\theta, \psi$.}
\end{algorithm}

\section{Implementation Details}
\label{app:implementation}

\subsection{SDXL Experiments}

\paragraph{Generator Architecture.}
SDXL~\citep{podell2023sdxl} is a latent diffusion model built on a UNet backbone with approximately 2.6B parameters. It employs a dual text-encoder design: a CLIP ViT-L encoder producing $d_1=768$-dimensional embeddings $c_{\mathrm{CLIP}}\in\mathbb{R}^{768}$ and an OpenCLIP ViT-bigG encoder producing $d_2=1280$-dimensional embeddings $c_{\mathrm{bigG}}\in\mathbb{R}^{1280}$; these are concatenated to form the conditioning vector
\begin{equation}
c = [c_{\mathrm{CLIP}};\,c_{\mathrm{bigG}}] \in \mathbb{R}^{2048}.
\end{equation}
Images $x\in\mathbb{R}^{3\times1024\times1024}$ are encoded by the SDXL VAE with 8$\times$ spatial downsampling into a 4-channel latent $z=\mathcal{E}(x)\in\mathbb{R}^{4\times128\times128}$, and decoded back via $\hat{x}=\mathcal{D}(z)$. The base UNet is trained natively at $1024\times1024$ with micro-conditioning on original and target image sizes to suppress training-resolution artifacts. SDXL supports an optional two-stage pipeline (base UNet for high-noise denoising, refiner UNet for low-noise enhancement); we use only the base UNet as the reference model $s_{\mathrm{ref}}$, which is queried with classifier-free guidance:
\begin{equation}
\tilde{s}_{\mathrm{ref}}(x_t,t,c) = s_{\mathrm{ref}}(x_t,t,\emptyset) + \alpha_{\mathrm{cfg}}\bigl(s_{\mathrm{ref}}(x_t,t,c)-s_{\mathrm{ref}}(x_t,t,\emptyset)\bigr).
\end{equation}

The forward VP-SDE corrupts data as
\begin{equation}
q_t(x_t\mid x_0) = \mathcal{N}\!\bigl(\alpha_t\,x_0,\,\sigma_t^2 I\bigr), \quad \alpha_t^2+\sigma_t^2=1,
\end{equation}
where $(\alpha_t,\sigma_t)$ follows the DDPM cosine schedule. The one-step generator $g_\theta$ shares the same UNet architecture and latent space, mapping a noise vector $z\sim\mathcal{N}(0,I)$ directly to a clean latent,
\begin{equation}
\hat{z}_0 = g_\theta(z,c) \in \mathbb{R}^{4\times128\times128},
\end{equation}
which is decoded to a $1024\times1024$ image by the frozen SDXL VAE. We initialize $g_\theta$ from the DMD2-SDXL-1step checkpoint~\citep{yin2024improved}. Both $s_{\mathrm{ref}}$ and the Teaching Assistant (TA) $s_\psi$ are initialized from the pretrained 50-step SDXL base model. We set $\sigma_{\mathrm{init}}=2.5$, the noise magnitude at which the generator input $z$ is injected into the forward diffusion during training, following Diff-Instruct~\citep{Luo2023DiffInstructAU} and SiD~\citep{zhou2024score}.

\paragraph{Training Setup.}
All models are trained on text prompts from LAION-Aesthetic-6.25+~\citep{zhou2024long} with no image data. We use PickScore~\citep{pickscore} as the reward model. Both $g_\theta$ and $s_\psi$ are optimized with Adam~\citep{kingma2014adam} ($\beta_1=0.0$, $\beta_2=0.999$), learning rate $5\times10^{-6}$, per-GPU batch size 1, and effective batch size 512 via gradient accumulation. The CFG scale for all reference model evaluations is $\alpha_{\mathrm{cfg}}=7.5$. The generator loss time weighting is $w(t)=1$ for all $t$~\citep{Luo2023DiffInstructAU}, and the TA DSM weighting $\lambda(t)$ follows the default SDXL training schedule. DRP hyperparameters are $K=4$, $S=4$, $\tau=0.01$. Training is conducted on 8 H100 GPUs (\textsc{Didr}: 48 hours; \textsc{Didr}$_{\mathrm{longer}}$: 72 hours).

\paragraph{Ablation Experiments.}
All ablation variants (Tables~\ref{tab:ablation_ks} and~\ref{tab:ablation_tau}) share the same setup as standard \textsc{Didr}: same initialization, optimizer, and training duration (48 hours on 8 H100 GPUs). Only the hyperparameter under study is varied; all others are fixed at the default values ($K=4$, $S=4$, $\tau=0.01$).

\subsection{Z-Image Experiments}

\paragraph{Generator Architecture.}
Z-Image~\citep{zimage} is a 6.15B-parameter text-to-image model built on a Scalable Single-Stream Diffusion Transformer (S3-DiT) with 30 transformer layers (hidden dimension 3840, 32 attention heads, FFN intermediate dimension 10240). Unlike dual-stream architectures, S3-DiT concatenates text tokens $h_{\mathrm{text}}$ (from a Qwen3-4B encoder), visual semantic tokens $h_{\mathrm{vis}}$ (from SigLIP features), and VAE image tokens $h_{\mathrm{img}}$ (encoded by the Flux VAE) into a single unified sequence
\begin{equation}
h = [h_{\mathrm{text}};\,h_{\mathrm{vis}};\,h_{\mathrm{img}}],
\end{equation}
which is processed jointly by all transformer layers, with 3D Unified RoPE for positional encoding and RMSNorm / QK-Norm / Sandwich-Norm for stability. In contrast to SDXL's VP-SDE, Z-Image adopts a rectified flow (flow matching) parameterization~\citep{liu2022flow}. The forward process interpolates linearly between data $x_0$ and noise $\epsilon\sim\mathcal{N}(0,I)$:
\begin{equation}
x_t = (1-t)\,x_0 + t\,\epsilon, \quad t\in[0,1],
\end{equation}
and the network is trained to predict the velocity field $v^*(x_t,t,c) = \epsilon - x_0$. Inference integrates the ODE
\begin{equation}
\frac{\diff x_t}{\diff t} = v_\phi(x_t,t,c)
\end{equation}
via Euler steps $x_{t+\Delta t} = x_t + \Delta t\, v_\phi(x_t,t,c)$, with the 50-step base model serving as $s_{\mathrm{ref}}$ (expressed in score form as $s_{\mathrm{ref}} = -v_\phi/\sigma_t$ under the flow parameterization). The base model supports arbitrary-resolution generation up to $1\text{k}\text{--}1.5\text{k}$ pixels via dynamic batch sizing.

The one-step generator $g_\theta$ maps noise $z\sim\mathcal{N}(0,I)$ directly to a clean image,
\begin{equation}
\hat{x}_0 = g_\theta(z,c),
\end{equation}
and is initialized from Z-Image-Turbo~\citep{zimage}, a consistency-distilled variant capable of high-quality generation in 8 steps. Initializing from Z-Image-Turbo rather than the raw 50-step base model provides a stronger single-step starting point, as it has already undergone consistency distillation. The frozen reference $s_{\mathrm{ref}}$ (equivalently $v_{\mathrm{ref}}$) and TA initialization are taken from the full 50-step Z-Image base model.

\paragraph{Training Setup.}
We follow the same data and reward setup as the SDXL experiments (LAION-Aesthetic-6.25+ prompts, PickScore reward). The optimizer is Adam ($\beta_1=0.0$, $\beta_2=0.999$) with a reduced learning rate of $1\times10^{-6}$ to accommodate the larger model scale. We use per-GPU batch size 1, effective batch size 256, and CFG scale $\alpha_{\mathrm{cfg}}=7.5$. DRP hyperparameters are $K=2$, $S=4$, $\tau=0.01$. Training is conducted on 8 H100 GPUs (96 hours).

\paragraph{DRP Posterior Sampling for Flow Matching.}
As derived in Appendix~\ref{app:proof_lem4}, diversity across the $K$ posterior chains is induced by randomized discretization: each chain draws $S$ timesteps uniformly at random from $(0,t_S]$ and runs the deterministic Euler ODE on its own schedule, yielding $K$ diverse approximate denoising endpoints.

\begin{table*}[!h]
\centering
\small
\setlength{\tabcolsep}{5pt}
\caption{Hyperparameter summary for all \textsc{Didr} model variants.}
\label{tab:hyperparams}
\begin{tabular}{lccc}
\toprule
\textbf{Hyperparameter} & \textbf{SDXL \textsc{Didr}} & \textbf{SDXL \textsc{Didr}$_{\text{longer}}$} & \textbf{zimage-\textsc{Didr}} \\
\midrule
Generator init          & DMD2-SDXL-1step    & DMD2-SDXL-1step    & Z-Image-Turbo \\
Reference diffusion     & SDXL (VP)          & SDXL (VP)          & Z-Image (flow) \\
Reward model            & PickScore          & PickScore          & PickScore \\
CFG scale $\alpha_{\mathrm{cfg}}$ & 7.5     & 7.5                & 7.5 \\
Posterior samples $K$   & 4                  & 4                  & 2 \\
Denoising steps $S$     & 4                  & 4                  & 4 \\
Temperature $\tau$      & 0.01               & 0.01               & 0.01 \\
Learning rate           & $5\times10^{-6}$   & $5\times10^{-6}$   & $1\times10^{-6}$ \\
Effective batch size    & 512                & 512                & 256 \\
Training cost           & 48 hours               & 72 hours               & 96 hours \\
\bottomrule
\end{tabular}
\end{table*}

\section{Image Generation Prompts}
\label{app:prompts}

The following lists the text prompts used to generate all qualitative figures in the paper.

\subsection*{Figure~\ref{fig:teaser} --- Teaser}
\begin{enumerate}
  \item A tranquil dormant mountain dusted with fresh snow under a clear pale blue daytime sky, calm still air, soft diffuse light, muted cool whites and blues, serene and peaceful alpine landscape
  \item Close-up macro of a human eye with a vivid blue iris, glittering sparkles and cosmic galaxy-like reflections, extreme detail, macro photography
  \item Portrait of a young African woman wearing a blue and yellow headwrap and a pearl earring, classical painting style, warm studio lighting, elegant and dignified
  \item A majestic tall sailing ship navigating through a glowing cosmic nebula, teal and cyan hues, dramatic fantasy scene, epic scale
  \item A vivid explosion of colorful paint splashing in mid-air, rainbow hues of orange, yellow, magenta, blue and purple, high-speed photography, dark background
  \item A tabby cat leaping through the air in a frozen action shot, photorealistic, natural daylight background
  \item An assorted sushi platter with rolls and nigiri, glowing neon blue ring lighting, futuristic and vibrant food photography
  \item A lush bouquet of deep red roses in a round green ceramic vase, soft natural window light, elegant floral still life
  \item A woman's portrait blending photorealism with blue watercolor ink splashes, artistic double-exposure style, ethereal and dreamy
\end{enumerate}

\subsection*{Figure~\ref{fig:qualitative} --- Qualitative Comparison}
SDXL comparison:
\begin{enumerate}
  \item Dramatic close-up portrait of a rugged elderly man with a weathered face, white beard, dark flat cap, overcast sea in the background
  \item A baby elephant in rain boots jumping in puddles under a rainbow, cheerful children's book illustration.
  \item An oil painting of a polar bear mother and cub drifting on an ice floe at Arctic dusk, warm amber sky.
\end{enumerate}

Z-Image comparison:
\begin{enumerate}
  \item A beautiful young woman in a white linen dress, Mediterranean coastal town background, warm golden hour light, photorealistic portrait
  \item A majestic white stallion galloping through crashing ocean waves at sunrise, cinematic photography.
  \item A young wizard girl casting a glowing spell in a floating sky library, anime key visual style.
\end{enumerate}

\subsection*{Figure~\ref{fig:tau_ablation} --- Temperature Ablation}
\begin{enumerate}
  \item Aurora borealis over a calm reflective lake surrounded by snow-covered pine trees, winter night
  \item Volcanic eruption with lava flows meeting snow and ice, dramatic fire and steam, epic landscape
  \item White arctic fox standing in snow, winter landscape, photorealistic wildlife photography
  \item Desert campsite at night with glowing orange tents and campfire, Milky Way visible overhead
\end{enumerate}

\subsection*{Figure~\ref{fig:onestep_vs_multistep} --- Multi-step Comparison}
\begin{enumerate}
  \item Japanese ukiyo-e woodblock print of Mount Fuji with cherry blossom trees and a lone figure on a wooden bridge
  \item Cute anthropomorphic fox in a cozy autumn caf\'{e}, illustrated anime art style, fallen leaves, warm tones
  \item Aerial drone view of a tropical volcanic island with a turquoise crater lake, surrounded by ocean
  \item Portrait of a Rajasthani woman with traditional Indian jewelry and a colorful headscarf, photorealistic
\end{enumerate}

\subsection*{Figure~\ref{fig:failure} --- Failure Cases}
\begin{enumerate}
  \item A woman placing a tray of food into a kitchen oven
  \item A large crowd of people standing in a queue outdoors
  \item A round wooden dining table in a bright room
  \item A person cutting raw meat on a kitchen counter surrounded by fresh vegetables
\end{enumerate}